\documentclass{article}

\PassOptionsToPackage{numbers, sort&compress}{natbib}
\usepackage[final]{neurips_2022}

\usepackage[utf8]{inputenc} %
\usepackage[T1]{fontenc}    %
\usepackage{hyperref}       %
\usepackage{url}            %
\usepackage{booktabs}       %
\usepackage{amsfonts}       %
\usepackage{nicefrac}       %
\usepackage{microtype}      %
\usepackage{xcolor}         %
\usepackage{amsthm}
\usepackage{amssymb}
\usepackage{cancel}
\usepackage{caption}
\usepackage{subcaption}
\usepackage{algorithm}%
\usepackage[noend]{algpseudocode}
\usepackage{wrapfig}
\usepackage{graphicx}
\usepackage{epigraph}
\usepackage[utf8]{inputenc} %
\usepackage[T1]{fontenc}    %
\usepackage{hyperref}       %
\usepackage{url}            %
\usepackage{booktabs}       %
\usepackage{amsfonts}       %
\usepackage{nicefrac}       %
\usepackage{microtype}      %
\usepackage{xcolor}         %

\usepackage{pifont}%
\newcommand{\cmark}{\ding{51}}
\newcommand{\xmark}{\ding{55}}

\usepackage{amsmath,amsfonts,bm}

\newcommand{\figleft}{{\em (Left)}}

\newcommand{\figright}{{\em (Right)}}

\def\eqref#1{equation~\ref{#1}}

\def\1{\bm{1}}

\DeclareMathAlphabet{\mathsfit}{\encodingdefault}{\sfdefault}{m}{sl}
\SetMathAlphabet{\mathsfit}{bold}{\encodingdefault}{\sfdefault}{bx}{n}

\def\gL{{\mathcal{L}}}

\def\gO{{\mathcal{O}}}

\newcommand{\E}{\mathbb{E}}

\newcommand{\Var}{\mathrm{Var}}

\usepackage{multirow}

\newtheorem{theorem}{Theorem}[section]

\newtheorem{lemma}[theorem]{Lemma}

\title{Mismatched No More: \\Joint Model-Policy Optimization for Model-Based RL}

\author{%
  Benjamin Eysenbach\thanks{Equal contribution.}\;\;$^{1 \, 2}$ \enskip Alexander Khazatsky$^{* \, 3}$ \enskip Sergey Levine$^{2 \, 3}$ \enskip Ruslan Salakhutdinov$^{1}$ \\
  \AND
  \\ \vspace{-3em} \\
  $^{1}$Carnegie Mellon University, \quad $^{2}$Google Brain, \quad $^{3}$UC Berkeley \\
  \texttt{beysenba@cs.cmu.edu},\quad \texttt{khazatsky@cs.stanford.edu}
}

\begin{document}
\maketitle

\begin{abstract}
Many model-based reinforcement learning (RL) methods follow a similar template: fit a model to previously observed data, and then use data from that model for RL or planning. However, models that achieve better training performance (e.g., lower MSE) are not necessarily better for control: an RL agent may seek out the small fraction of states where an accurate model makes mistakes, or it might act in ways that do not expose the errors of an inaccurate model. As noted in prior work, there is an objective mismatch: models are useful if they yield good policies, but they are trained to maximize their accuracy, rather than the performance of the policies that result from them.  In this work, we propose a single objective for jointly training the model and the policy, such that updates to either component increase a lower bound on expected return. To the best of our knowledge, this is the first lower bound for model-based RL that holds globally and can be efficiently estimated in continuous settings; it is the only lower bound that mends the objective mismatch problem. A version of this bound becomes tight under certain assumptions. Optimizing this bound resembles a GAN: a classifier distinguishes between real and fake transitions, the model is updated to produce transitions that look realistic, and the policy is updated to avoid states where the model predictions are unrealistic. Numerical simulations demonstrate that optimizing this bound yields reward maximizing policies and yields dynamics that (perhaps surprisingly) can aid in exploration. We also show that a deep RL algorithm loosely based on our lower bound can achieve performance competitive with prior model-based methods, and better performance on certain hard exploration tasks.

\end{abstract}

\section{Introduction}

Much of the appeal of model-based RL is that model learning is a simple and scalable supervised learning problem.
However, even after learning a very accurate model, it is hard to say whether that model will actually be useful for model-based RL~\citep{farahmand2017value, lambert2020objective}. For example, a model might make small mistakes in critical states that cause a policy to take suboptimal actions. Alternatively, a model with large errors may yield a policy that attains high return if the model errors occur in states that the policy never visits.

The underlying problem is that dynamics models are trained differently from how they are used.
Typical model-based methods train a model using data sampled from the \emph{real} dynamics (e.g., using maximum likelihood), but apply these models by using data sampled from the \emph{learned} dynamics~\citep{deisenroth2011pilco, williams2017mpc, janner2019trust, hafner2019dream}.
Prior work has identified this \emph{objective mismatch} issue~\citep{farahmand2017value, luo2018algorithmic, lambert2020objective}: the model is trained using one objective, but the policy is trained using a different objective. Designing an objective for model training that is guaranteed to improve the expected reward remains an open problem. So, \emph{how should we train a dynamics model so that it produces high-return policies when used for model-based~RL?}

The key idea in this paper is to view model-based RL as a latent-variable problem: the latent variable is the trajectory and the cumulative reward is interpreted as the probability that the trajectory solves the task. Inferring the latent variable corresponds to learning a dynamics model. Latent variable models are typically learned via an evidence lower bound, and we show how a similar evidence lower bound provides yields a new objective for model-based RL. In the same way that the evidence lower bound is a joint optimization problem over two variables, our objective will \emph{jointly} optimize the model and the policy using the same objective: to produce realistic and high-return trajectories. Our objective differs from standard model-based RL objectives, where it is more common to pit the model \emph{against} the policy~\citep{bagnell2001solving, nilim2003robustness, ross2012agnostic}.
A consequence of maximizing the lower bound is that the dynamics model does not learn the true dynamics, but rather learns optimistic dynamics that facilitate exploration.

The main contribution of this work is an objective for model-based RL. To the best of our knowledge, this is the first lower bound for model-based RL that holds globally (unlike~\citet{luo2018algorithmic}) and can be efficiently estimated in continuous settings (unlike~\citet{kearns2002near}). It is the first lower bound that jointly optimizes the model and policy using the same objective. We also present a more complex version of this bound that becomes tight under some assumptions. Through numerical simulations in simple tasks, we demonstrate that optimizing the bound yields reward-maximizing policies and yields an optimistic dynamics model that can aid exploration. We also demonstrate that our bound gracefully accounts for function approximation error in the model -- the policy is penalized for taking transitions that the model cannot represent. Finally, we show that we can use parts of our theoretically-motivated objective to design a practically-applicable deep RL method that, despite deviating from the theory, can match the performance of prior model-based methods.

\section{Related Work}
\label{sec:related-work}

Most model-based RL methods use maximum likelihood to fit the dynamics model and then use RL to maximize the expected return under samples from that model~\citep{deisenroth2011pilco, williams2017mpc, chua2018deep, hafner2019dream, janner2019trust}. 
As noted in prior work, this maximum likelihood objective is not aligned with the RL objective: models that achieve higher likelihood do not necessarily produce better policies~\citep{ziebart2010modeling, talvitie2014model, farahmand2017value, luo2018algorithmic, lambert2020objective}. This issue is referred to as the \emph{objective mismatch} problem: the model and policy (or planner) are optimized using different objectives. This problem arises in almost all model-based RL approaches, including those that train the model to predict the value function~\citep{oh2017value, schrittwieser2020mastering} or that perform planning~\citep{chua2018deep, schrittwieser2020mastering}.

Some prior work addresses this problem by decreasing the rewards at states where the model is inaccurate~\citep{sorg2010internal, Kidambi2020MOReLM, yu2020mopo, yu2021combo, luo2018algorithmic}, a strategy that our method will also employ through a discriminator. While some of these methods also use discriminators in this manner, ours is the first to optimize a lower bound on returns. Other work modifies the model objective to include multi-step rollouts~\citep{joseph2013reinforcement, talvitie2014model, venkatraman2016improved, farahmand2017value, asadi2018lipschitz, asadi2019combating}. Our method will modify the model objective in a different way, so that the model objective is exactly the same as the policy objective.
Some prior work directly optimizes the model to produce good policies~\citep{okada2017path, amos2018differentiable, srinivas2018universal, d2020gradient, nikishin2021control}, as theoretically analyzed in~\citet{grimm2020value}. While our aim is the same as these prior methods, our approach will not require differentiating through unrolled model updates or optimization procedures.

\begin{wraptable}[10]{r}{0.67\textwidth}
    \centering
    \vspace{-0.5em}
    \caption{\footnotesize \textbf{Lower bounds for model-based RL.}}
    \vspace{-0.5em}
    {\footnotesize
    \begin{tabular}{c|p{1.3cm}|p{1.5cm}|p{1cm}|p{1cm}}
        & \citet{luo2018algorithmic} & \citet{kearns2002near} & MnM (Eq.~\ref{eq:lb}) & MnM (Eq.~\ref{eq:lb2})\\ \midrule
        holds globally & \xmark & \cmark & \cmark & \cmark \\
        efficient to compute & \cmark & \xmark & \cmark & \cmark \\
        unified objective & \xmark & \xmark & \cmark & \cmark\\
        tight at optimality & \xmark$^\dagger$ & \cmark & \xmark & \cmark \\
        \multicolumn{5}{l}{\hspace{6em}\scriptsize $^\dagger$ Discrepancy measure is non-zero in stochastic environments.}
    \end{tabular}
    }
    \label{tab:prior-work}
\end{wraptable}

Our work builds on prior lower bounds for model-based RL. \citet{kearns2002near} provide a lower bound that holds globally, but that is only computable in tabular settings. \citet{luo2018algorithmic} provide a lower bound that can be efficiently estimated, but which only holds for nearby policies and models. As shown in Table~\ref{tab:prior-work}, our bound combines the strengths of these prior works, providing a lower bound that holds globally and can be efficiently estimated in MDPs with continuous states and actions. Unlike prior work, our bound also mends the objective mismatch problem.

Our theoretical derivation builds on prior work that casts model-based RL as a two-player game between a model-player and a policy-player~\citep{bagnell2001solving, nilim2003robustness, ross2011dagger, rajeswaran2020game}. Whereas prior work pits model and policy against one another, our formulation will result in a cooperative game: the model and policy cooperate to optimize the \emph{same} objective (a lower bound on the expected return). Our approach, though structurally resembling a GAN, is different from prior work that replaces a maximum likelihood model with a GAN model~\citep{bai2019model, chen2020gan, kurutach2018learning}.

The most similar prior work is VMBPO~\citep{chow2020variational}, which also jointly optimizes the model and the policy using the same objective.  However, while our objective is a lower bound on expected return, VMBPO maximizes a different, risk-seeking objective, which is an \emph{upper} bound on expected return (see Appendix~\ref{appendix:vmbpo}). This different objective can be expressed as the expected return plus the variance of the return, so VMBPO has the undesirable property of preferring policies that receive slight lower return if the variance of the return is much larger (see Appendix~\ref{appendix:vmbpo}). Indeed, while most of the components of our method (e.g., classifiers, GAN-like models) have been used in prior work, our paper is the first to provide a precise recipe for combining these components into an objective that is a provable lower bound.

\section{A Unified Objective for Model-Based RL}
\label{sec:lower-bound}

\paragraph{Notation.}
We focus on the Markov decision process with states $s_t$, actions $a_t$, initial state distribution $p_0(s_0)$, positive reward function $r(s_t, a_t) > 0$ , and dynamics $p(s_{t+1} \mid s_t, a_t)$. Our aim is to learn a control policy $\pi_\theta(a_t \mid s_t)$ with parameters $\theta$ that maximizes the expected discounted return:
\begin{equation}
    \max_\theta \E_{\pi_\theta}\bigg[\sum_{t=0}^\infty \gamma^t r(s_t, a_t) \bigg]. \label{eq:standard-obj}
\end{equation}
We use transitions $(s_t, a_t, r_t, s_{t+1})$ collected from the (real) environment to train the dynamics model $q_\theta(s_{t+1} \mid s_t, a_t)$, and use transitions sampled from this learned model to train the policy.
To simplify notation, we will define a trajectory $\tau \triangleq (s_0, a_0, s_1, a_1, \cdots)$ as a sequence of states and actions visited in an episode. We define $R(\tau) \triangleq \sum_{t=0}^\infty \gamma^t r(s_t, a_t)$ as the discounted return of a trajectory. We define two distributions over trajectories: $p^\pi(\tau)$ and $q^\pi(\tau)$ for when policy $\pi_\theta$ interacts with dynamics $p(s_{t+1} \mid s_t, a_t)$ and $q_\theta(s_{t+1} \mid s_t, a_t)$, respectively:
\begin{align*}
    p^\pi(\tau) &= p_0(s_0) \prod_{t=0}^\infty p(s_{t+1} \mid s_t, a_t) \pi_\theta(a_t \mid s_t), \quad q^\pi(\tau) = p_0(s_0) \prod_{t=0}^\infty q_\theta(s_{t+1} \mid s_t, a_t) \pi_\theta(a_t \mid s_t).
\end{align*}

\paragraph{Desiderata.}
Our aim is to design an objective $\gL(\theta)$ with two properties. \emph{First}, this objective should be a lower bound on the expected return in the true environment: if the policy does well in the learned model, we are guaranteed that the policy will also do well in the true environment.
The expected return under the learned model, which most prior model-based RL methods use to train the policy, is not a lower bound on the expected return~\citep{kearns2002near, luo2018algorithmic}.
While prior work has made strides in developing lower bounds for model-based RL, even the best lower bounds do not hold for all models~\citep{luo2018algorithmic} or are limited to tabular settings~\citep{kearns2002near}.

\emph{Second}, this objective should be the same for the policy and the dynamics model, such that updates to the model would improve the policy, and vice versa. This desiderata is important because prior work has found that training the model to be more accurate (increase likelihood) can decrease the policy's expected return under that model~\citep{ziebart2010modeling, talvitie2014model, farahmand2017value, luo2018algorithmic, lambert2020objective}.
One prior method (VMBPO~\citep{chow2020variational}) does train the model and policy with the same objective, but this objective corresponds to a risk-seeking, upper-bound on the expected returns.

\paragraph{An objective for model-based RL.}
We now introduce an objective that achieves these aims.
The key idea to deriving this result is to take a probabilistic perspective on decision making: we view the trajectory as an unobserved random variable, and the reward function as the probability that a trajectory solves a task. Then, the problem of inferring this random variable is equivalent to learning the dynamics model. Using an evidence lower bound, we obtain an objective for jointly training the model and policy. We provide the full derivation in Appendix~\ref{proof:lemma-main-result}.

Our resulting objective is the policy's reward when interacting with the learned model, but using a new reward function. The new reward function combines the task reward with a term that measures the difference between the learned model and the real environment. We define our objective
\begin{align}
    \gL(\theta) \triangleq \E_{q^{\pi_\theta}(\tau)} \bigg[\sum_{t=0}^\infty \gamma^t \tilde{r}(s_t, a_t, s_{t+1}) \bigg], \label{eq:lb}
\end{align}
where the modified reward function is defined as
\footnotesize \begin{align}
\tilde{r}(s_t, a_t, s_{t+1}) \triangleq (1 - \gamma) \log r(s_t, a_t) + \log \left( \frac{p(s_{t+1} \mid s_t, a_t)}{q(s_{t+1} \mid s_t, a_t)} \right) - (1 - \gamma) \log (1 - \gamma). \label{eq:aug-reward}
\end{align} \normalsize

Intuitively, the new reward function $\tilde{r}$ penalizes the policy for taking transitions that are unlikely under the true dynamics model, similar to prior work~\citep{vemula2020planning, eysenbach2020off, yu2021combo}. Later, we will show that we can estimate this augmented reward \emph{without knowing the true environment dynamics} by using a classifier.

We will optimize this lower bound with respect to both the policy $\pi_\theta(a_t \mid s_t)$ and the dynamics model $q_\theta(s_{t+1} \mid s_t, a_t)$. For the policy, we maximize the modified reward using samples from the learned model; the only difference from prior work is the modification to the reward function. Training the dynamics model using this objective is very different from standard maximum likelihood training. The model is optimized to sample trajectories that are similar to real dynamics (like a GAN) and that have high reward (unlike a GAN). This objective differs from VMBPO~\citep{chow2020variational} by taking the $\log(\cdot)$ of the original reward functions; our experiments demonstrate that excluding this component invalidates our lower bound and results in learning suboptimal policies (Fig.~\ref{fig:gridworld-line}).

Our objective achieves the two desiderata. 
\emph{First}, our objective is a lower bound on the expected return. To state this result formally, we will take the logarithm of the expected return. Of course, maximizing the $\log(\cdot)$ of the expected return is equivalent to maximizing the expected return.
\begin{theorem} \label{lemma:main-result}
The following bound holds for \textbf{any} dynamics $q(s_{t+1} \mid s_t, a_t)$ and policy $\pi(a_t \mid s_t)$:
\begin{align*}
    \log \E_{\pi}\bigg[\sum_{t=0}^\infty \gamma^t r(s_t, a_t) \bigg] &\ge \gL(\theta).
\end{align*}
\end{theorem}
The proof is presented in Appendix~\ref{proof:lemma-main-result}.
To the best of our knowledge, this is the first global (unlike~\citet{luo2018algorithmic}) and efficiently-computable (unlike~\citet{kearns2002near}) lower bound for model-based RL. The model and the policy are trained using the same objective: updating the model not only increases the objective for the model, but also increases the objective for the policy.

Sec.~\ref{sec:alg} will introduce an algorithm to maximize this lower bound. While this lower bound may not be tight, experiments in Sec.~\ref{sec:experiments} demonstrate that optimizing this lower bound yields policies that achieve high reward across a wide range of tasks. In Appendix~\ref{appendix:tight}, we propose a variant of this bound that does become tight at convergence. Because this tight bound is more complex, we focus on the simple bound in this paper.

\paragraph{The optimal dynamics are optimistic.}
We now return to analyzing the simpler lower bound ($\gL(\theta)$ in~Eq.~\ref{eq:lb}). In stochastic environments, the dynamics model that optimizes this lower bound is not equal to the true environment dynamics. Rather, it is biased towards sampling trajectories with high return. Ignoring parametrization constraints, the dynamics model that optimizes our lower bound is $q^*(\tau) = \frac{p(\tau)R(\tau)}{\int p(\tau') R(\tau') d\tau'}$ (proof in Appendix~\ref{proof:lemma-tight}.).
While it may seem surprising that the objective-optimizing dynamics would differ from the true dynamics, this result is analogous to a VAE, where the ELBO-optimizing encoder differs from the prior.
The optimism in the dynamics model may accelerate policy optimization, a hypothesis we will test in Sec.~\ref{sec:didactic}.

Would the optimistic dynamics overestimate the policy's return, violating Theorem~\ref{lemma:main-result}? While using the optimistic dynamics with the \emph{original} reward function will overestimate the true return, using the optimistic dynamics with the \emph{augmented} reward function yields a valid lower bound. We demonstrate this effect in Fig.~\ref{fig:lower-bounds}.

\section{Practical Optimization of the Lower Bound}
\label{sec:alg}

The previous section presented a single (global) lower bound ($\gL$ from Eq.~\ref{eq:lb}) for jointly optimizing the policy and the dynamics model. In this section, we develop a practical %
algorithm for optimizing this lower bound. The main challenge in optimizing this bound is that the augmented reward function depends on the transition probabilities of the real environment, $p(s_{t+1} \mid s_t, a_t)$, which are unknown. We address this challenge by learning a classifier (Sec.~\ref{sec:classifier}). Of course, Theorem~\ref{lemma:main-result} only holds if the classifier is Bayes-optimal. We then describe the precise update rules for the policy, dynamics model, and classifier (Sec.~\ref{sec:updates}).

\subsection{Estimating the Augmented Reward Function}
\label{sec:classifier}

To estimate the augmented reward function, which depends on the transition probabilities of the real environment, we learn a classifier that distinguishes real transitions from fake transitions. This approach is similar to GANs~\citep{goodfellow2014gan} and similar to prior work in RL~\citep{eysenbach2020off, yu2021combo}. We use \mbox{$C_\phi(s_t, a_t, s_{t+1}) \in [0, 1]$} to denote the classifier, which we train to distinguish real versus model transitions using the standard cross entropy loss:
\begin{align}
      \max_\phi \gL_C(s_t^\text{real}, a_t^\text{real}, s_{t+1}^\text{real}, s_{t+1}^\text{model}; \phi) \triangleq \log C_\phi(s_t^\text{real}, a_t^\text{real}, s_{t+1}^\text{real}) + \log \left( 1 - C_\phi(s_t^\text{real}, a_t^\text{real}, s_{t+1}^\text{model}) \right).
      \label{eq:classifier}
\end{align}
Note that the real transition $(s_t^\text{real}, a_t^\text{real}, s_{t+1}^\text{real})$ and model transition $(s_t^\text{real}, a_t^\text{real}, s_{t+1}^\text{model})$ have the same initial state and initial action.
Once trained, we can use the classifier's predictions to estimate the augmented reward function:
\begin{equation}
    \text{\footnotesize$\tilde{r}(s_t, a_t, s_{t+1}) \approx \log r(s_t, a_t) + \log \left(\frac{C_\phi(s_t, a_t, s_{t+1})}{1 - C_\phi(s_t, a_t, s_{t+1})} \right).$} \label{eq:augmented-reward}
\end{equation}
The approximation above reflects function approximation error in learning the classifier. Now that we can estimate the augmented reward function, we can apply \emph{any} RL method to maximize the augmented reward under transitions sampled from the dynamics model. The following section describes a particular instantiation using an off-policy RL algorithm.

\begin{wrapfigure}{R}{0.5\textwidth}
    \centering
    \vspace{-3em}
    \includegraphics[width=\linewidth]{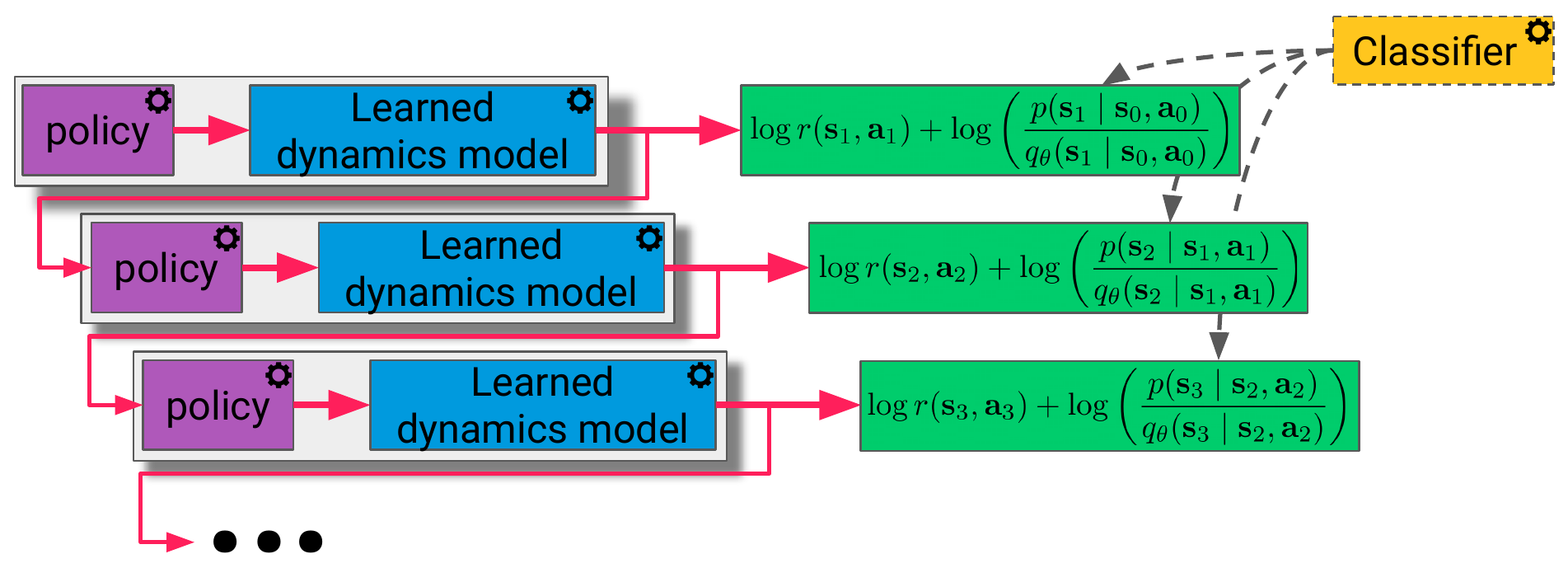}
    \vspace{-1em}
    \caption{{\footnotesize \textbf{Mismatched No More} is a model-based RL algorithm that learns a policy, dynamics model, and classifier. The classifier distinguishes real transitions from model transitions.
    The policy and dynamics are jointly optimized to sample transitions that yield high returns and look realistic, as estimated by the classifier.}} \label{fig:method}
    \vspace{-1.em}
\end{wrapfigure}

\subsection{Updating the Model and Policy}
\label{sec:updates}

We now present our complete method, which trains three components: a classifier, a policy, and a dynamics model. Our method alternates between \emph{(1)} updating the policy (by performing RL using model experience with augmented rewards) and \emph{(2)} updating the dynamics model and classifier (using a GAN-like objective). In describing the loss functions below, we use the superscripts $(\cdot)^\text{real}$ and $(\cdot)^\text{model}$ to denote transitions that have been sampled from the true environment dynamics or the learned dynamics function. To reduce clutter, we omit the superscripts when unambiguous.

\paragraph{Updating the policy.}
The policy is optimized to maximize the augmented reward on transitions sampled from the learned dynamics model. While this optimization can be done using any RL algorithm, including on-policy methods, we will focus on an off-policy actor-critic method.

We define the Q function as sum of \emph{augmented} rewards under the learned dynamics model:
{\footnotesize \begin{equation}
     Q(s_t, a_t) \triangleq \E_{\substack{\pi(a_t \mid s_t),\\q_\theta(s_{t+1 \mid s_t, a_t)}}}\left[\sum_{t' = t}^\infty \gamma^{t' - t} \tilde{r}(s_{t'}, a_{t'}) \mid \substack{s_t=s_t,\\ a_t=a_t}\right].
\end{equation}}\!\!
We approximate the Q function using a neural network $Q_\psi(s_t, a_t)$ with parameters $\phi$.
We train the Q function using the TD loss on transitions sampled from the \emph{learned} dynamics model:
{\footnotesize \begin{equation}
    \gL_Q(s_t, a_t, r_t, s_{t+1}^\text{model}; \psi) = \left(Q_\psi(s_t, a_t) - \lfloor y_t \rfloor_\text{sg} \right)^2, \label{eq:q-loss}
\end{equation}}\!\!
where $\lfloor \cdot \rfloor_\text{sg}$ is the stop-gradient operator and $y_t \!=\! \tilde{r}(s_t, a_t, s_{t+1}^\text{model}) \!+\! \gamma \E_{\pi(a_{t+1} \mid s_{t+1}^\text{model})}\left[Q_\psi(s_{t+1}^\text{model}, a_{t+1}) \right]$ is the TD target. The augmented reward $\tilde{r}$ is estimated using the learned classifier (Eq.~\ref{eq:augmented-reward}). To estimate the corresponding value function, we use a 1-sample approximation: $V_\psi(s_t) = Q_\psi(s_t, a_t)$ where $a_t \sim \pi_\theta(a_t \mid s_t))$. The policy is trained to maximize the Q function:
{\footnotesize \begin{equation}
    \max_\theta \gL_{\pi}(s_t; \theta) \triangleq \E_{\pi_\theta(a_t \mid s_t)}\left[Q_\psi(s_t, a_t)\right]. \label{eq:pi-loss}
\end{equation}}\!\!
In our implementation, we regularize the policy by adding an additional entropy regularizer.
Following prior work~\citep{fujimoto2018td3}, we maintain two Q functions and two target Q functions, using the minimum of the two to compute the TD target.
See Appendix~\ref{appendix:details} for details.

\begin{figure}[t]
\vspace{-1em}
\begin{algorithm}[H]
{\footnotesize
\caption{{\footnotesize \textbf{Mismatched no More (MnM)} is an algorithm for model-based RL. The method alternates between training the policy on experience from the learned dynamics model with augmented rewards and updating the model+classifier using a GAN-like loss. 
While we use an off-policy RL algorithm on L\ref{line:rl}, any other RL algorithm can be substituted.}}\label{alg:method}
\begin{algorithmic}[1]
\While{not converged}
\State Sample experience from the learned model.
\State Modify rewards using the classifier (Eq.~\ref{eq:augmented-reward}).
\State Update the policy and Q function using the model experience and modified rewards (Eq.s~\ref{eq:pi-loss} and~\ref{eq:q-loss}). \label{line:rl}
\State Update model and classifier using GAN-like losses (Eq.s~\ref{eq:classifier} and~\ref{eq:model-objective}).
\State (Infrequently) Sample experience from the real model.
\EndWhile
\State \textbf{return} policy $\pi_\theta(a_t \mid s_t)$.
\end{algorithmic}}
\end{algorithm}
\vspace{-2em}
\end{figure}

\paragraph{Updating the dynamics model.}
To optimize the dynamics model, we rewrite the lower bound in terms of a single transition (derivation in Appendix~\ref{appendix:single-transition}):
{\footnotesize \begin{align}
    \gL_q(s_t^\text{real}, a_t^\text{real}; \theta) = \E_{s_{t+1}^\text{model} \sim q_\theta(s_{t+1} \mid s_t^\text{real}, a_t^\text{real})} \bigg[V_\psi(s_{t+1}^\text{model}) + \log \left(\frac{C_\phi(s_t^\text{real}, a_t^\text{real}, s_{t+1}^\text{model})}{1 - C_\phi(s_t^\text{real}, a_t^\text{real}, s_{t+1}^\text{model})} \right) \bigg]. \label{eq:model-objective}
\end{align}}\!\!
This loss is an approximation of our original lower bound (Eq.~\ref{eq:lb}) because we estimate the difference in dynamics using a classifier.  This approximation is standard in prior work on GANs~\citep{goodfellow2014gan} and adversarial inference~\citep{donahue2016adversarial, dumoulin2016adversarially}.

Intuitively, the procedure for optimizing the dynamics model and the classifier resembles a GAN~\citep{goodfellow2014gan}: the classifier is optimized to distinguish real transitions from model transitions, and the model is updated to fool the classifier (and increase rewards). However, \emph{our method is not equivalent to simply replacing a maximum likelihood model with a GAN model}. Indeed, such an approach would not optimize a lower bound on expected return. Rather, our model objective includes an additional value term and our policy objective includes an additional classifier term. These changes enable the model and policy to optimize the same objective, which is a lower bound on expected return.

\paragraph{Algorithm summary.}

We summarize the method in Alg.~\ref{alg:method} and provide an illustration in Fig.~\ref{fig:method}. We call the method \textsc{Mismatched no More} (MnM) because the policy and model optimize the same objective, thereby resolving the objective mismatch problem noted in prior work. While the model and policy are optimized using the same objective, that objective can stop being a lower bound if the learned classifier is not Bayes-optimal.

Implementing MnM on top of a standard model-based RL algorithm is straightforward. First, create an additional classifier network. Second, instead of using the maximum likelihood objective to train the model, use the GAN-like objective in Eq.~\ref{eq:model-objective} to update both the model and the classifier. Third, add the classifier's logits to the predicted rewards (Eq.~\ref{eq:augmented-reward}).
Following prior work~\citep{janner2019trust}, we learn a neural network to predict the true environment rewards. %

\subsection{A Note about Exploration}
\label{sec:exploration}
The classifier term in the augmented reward (Eq.~\ref{eq:aug-reward}) is a double edged sword. Our theoretical derivation suggests that this term should appear, and our didactic experiments demonstrate that removing this term results in suboptimal behavior. Experiments also show that this term can effectively combat errors in the learned model. Prior work in offline RL has found similar model-error reward terms critical for achieving good performance in the offline setting~\citep{sorg2010internal, Kidambi2020MOReLM, yu2020mopo, yu2021combo}.

However, when scaling MnM to continuous control tasks in the online setting, we found that including this term hurts performance (Fig.~\ref{fig:no-classifier}). This makes sense: this classifier term is exactly the opposite of exploration objectives based on model error (e.g.,~\cite{stadie2015incentivizing}), and penalizes the policy for performing exploration. Our continuous control experiments in Sec.~\ref{sec:comparisons} will deviate from our theoretical derivation: they will use the model objective suggested by our theory, but the same policy objective as prior. That is, the policy maximizes $r(s_t, a_t)$ instead of $\tilde{r}(s_t, a_t)$ (Eq.~\ref{eq:aug-reward}). We call this method ``MnM-approx.'' We note that many theoretical model-based RL papers also find it necessary to implement practical algorithms that differ from the theory~\citep{luo2018algorithmic, vemula2020planning, yu2021combo}.

\section{Numerical Simulations}
\label{sec:experiments}

The primary aim of our experiments is to verify that our objective is a valid lower bound, and that maximizing the bound yields reward-maximizing policies. We study these questions in Sec.~\ref{sec:didactic}, where we will use tabular problems so that we can study the objective in the absence of function approximation error. Then, in Sec.~\ref{sec:comparisons},
we show adapting part of our objective into a scalable model-based RL algorithm yields a method that, while deviating from the theory, can achieve competitive results on benchmark tasks.

\subsection{Understanding the Lower Bound and the Learned Dynamics}
\label{sec:didactic}

Our didactic experiments test whether optimizing the lower bound produces optimal policies and study how the components of the lower bound. We use tabular domains in this section, as they allow us to analytically compute the optimal policy for comparison, and allow us to evaluate the lower bound in the absence of function approximation error. 

Our first experiment compares the lower bound to an oracle method that applies Q-learning to a perfect dynamics model. In comparison to this baseline, MnM learns a dynamics model using the GAN-like objective (Eq.~\ref{eq:model-objective}) and maximizes a modified reward function (Eq.~\ref{eq:augmented-reward}). We also compare to VMBPO, which learns a dynamics model similar to MnM but omits the log-transformation of the reward function; this transformation ordinarily encourages pessimistic behavior. For this experiment, we use a $10 \times 10$ gridworld with stochastic dynamics and sparse reward, shown in Fig.~\ref{fig:stochastic-gridworld} \figleft. The results, shown in Fig.~\ref{fig:stochastic-gridworld}, show that MnM outperforms both Q-learning with the correct model and VMBPO on this task. We hypothesize that VMBPO performs poorly on this task because it maximizes an upper bound on performance; and confirm this in the following experiments.

\begin{figure}[t]
    \centering
    \vspace{-2em}
    \begin{subfigure}[b]{0.54\linewidth}
        \centering
      \begin{subfigure}[c]{0.65\linewidth}
        \centering
          \includegraphics[width=\linewidth]{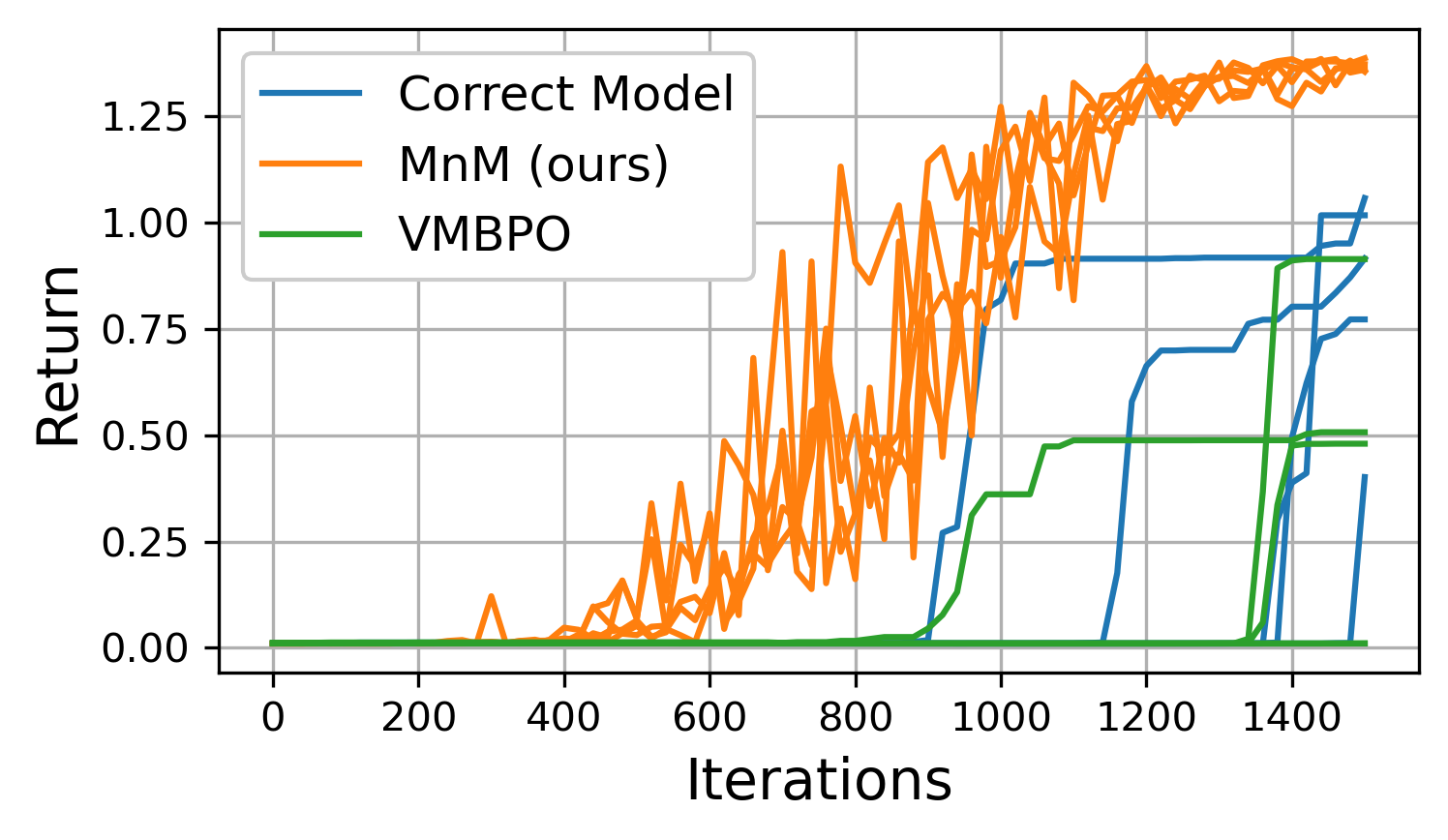}%
      \end{subfigure}%
      ~
      \begin{subfigure}[c]{0.35\linewidth}
        \vspace{-1em}
          \includegraphics[width=\linewidth]{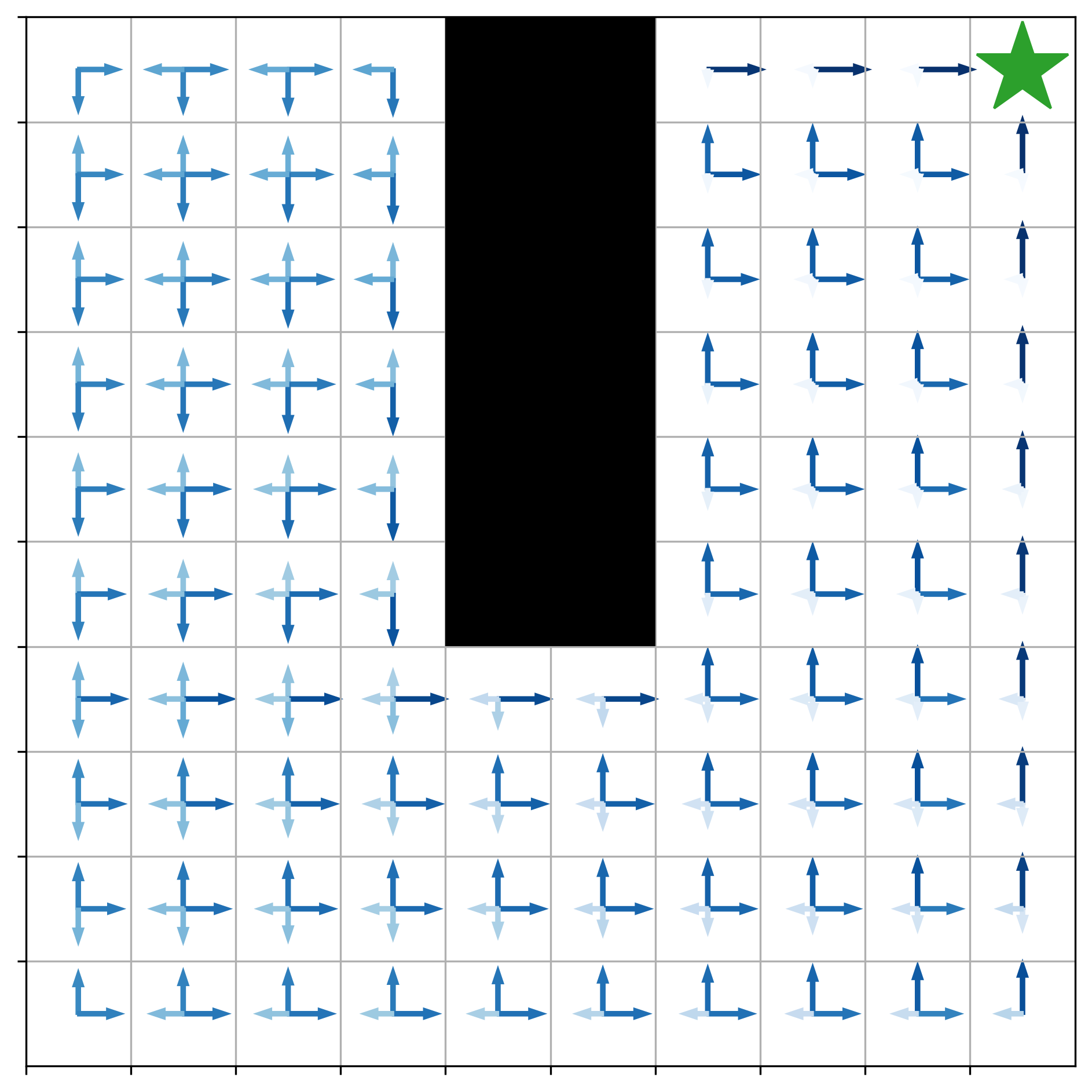}
      \end{subfigure}
      \vspace{-1em}
      \caption{{\footnotesize Stochastic Gridworld}} \label{fig:stochastic-gridworld}
  \end{subfigure}
  \hfill
  \begin{subfigure}[b]{0.44\linewidth}
        \centering
        \includegraphics[width=\linewidth]{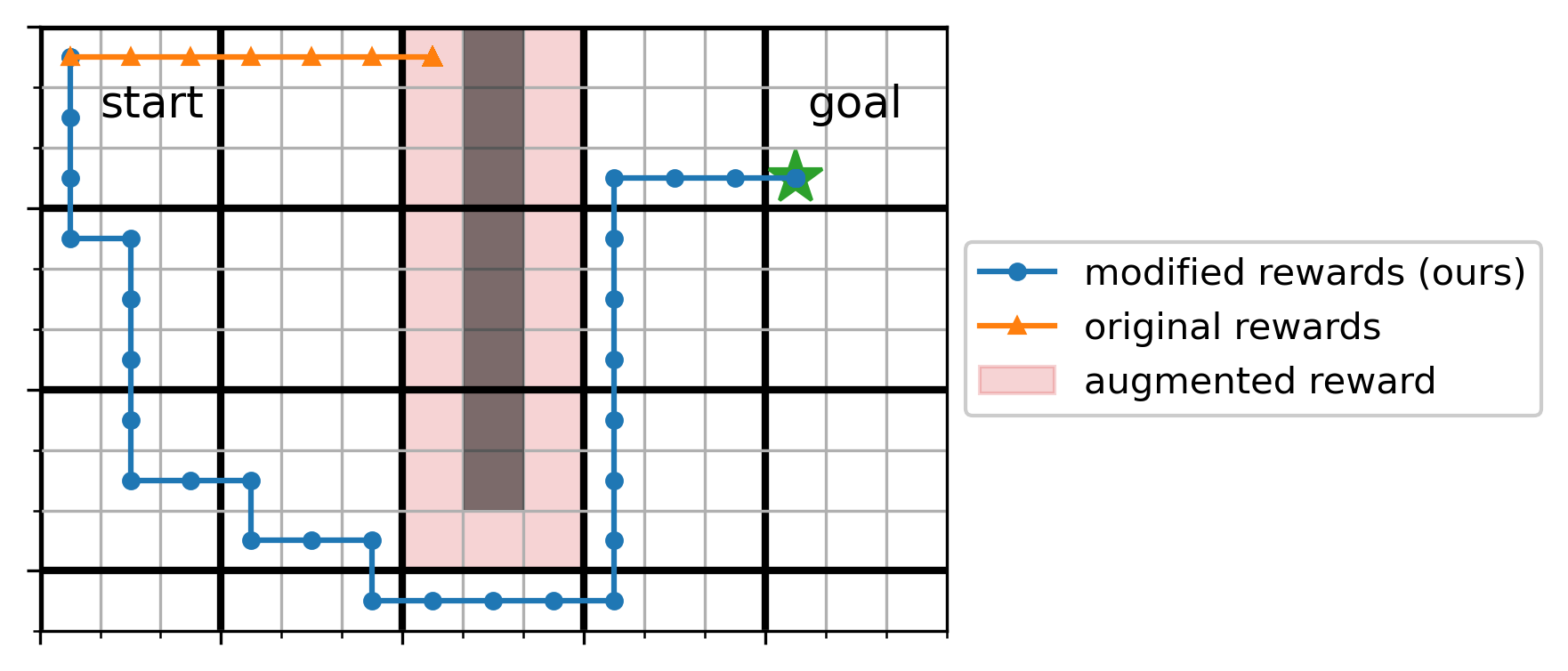}
        \vspace{-1.em}
        \caption{\footnotesize Inaccurate Models} 
        \label{fig:aliasing}
      \end{subfigure}%
      \caption{\footnotesize \textbf{Gridworld experiments.}
      \figleft \,  We apply MnM to a navigation task with transition noise that moves the agent to neighboring states with equal probability. MnM solves this task more quickly than Q-learning and VMBPO. The dynamics learned by MnM are different from the real dynamics, changing the transition noise (blue arrows) to point towards the goal.
      \figright \, We simulate function approximation by a learning model that makes the same predictions for groups of $3 \times 3$ states, resulting in a model that is inaccurate around obstacles. The modified reward compensates for these errors by penalizing the policy for navigating near obstacles.}
\end{figure}

We hypothesize that MnM outperforms Q-learning with the correct model because it learns ``optimistic'' dynamics. We test this hypothesis by visualizing the dynamics model learned by MnM (Fig.~\ref{fig:stochastic-gridworld} \figleft). While the true environment dynamics have stochasticity that moves the agent in a random direction with \emph{equal} probability, the MnM dynamics model biases this stochasticity to lead the agent towards the goal (blue arrows point towards the goal). Of course, we use the true environment dynamics, not the optimistic dynamics model, for evaluating the policies.

Our augmented reward function contains two crucial components, \emph{(1)} the classifier term and \emph{(2)} the logarithmic transformation of the reward function. We test the importance of the classifier term in correcting for inaccurate models. To do this, we limit the capacity of the MnM dynamics model so that it makes ``low-resolution'' predictions, forcing all states in $3 \times 3$ blocks to have the same dynamics. We will use the gridworld shown in Fig.~\ref{fig:aliasing} \figleft, which contains obstacles that occur at a finer resolution than the model can detect.
When the ``low resolution'' dynamics model makes predictions for states near the obstacle, it will average together some states with obstacles and some states without obstacles. Thus, the model will (incorrectly) predict that the agent always has some probability of moving in each direction, even if that direction is actually blocked by an obstacle. However, the classifier (whose capacity we have also limited) detects that the dynamics model is inaccurate in these states, so the augmented reward is much lower at these states. Thus, MnM is able to solve this task despite the inaccurate model; an ablation of MnM that removes the classifier term  attempts to navigate through the wall and fails to reach the goal.

\begin{figure}[t]
    \centering
    \begin{subfigure}[b]{0.55\textwidth}
        \centering
        \begin{subfigure}[c]{0.4\linewidth}
            \includegraphics[width=\linewidth]{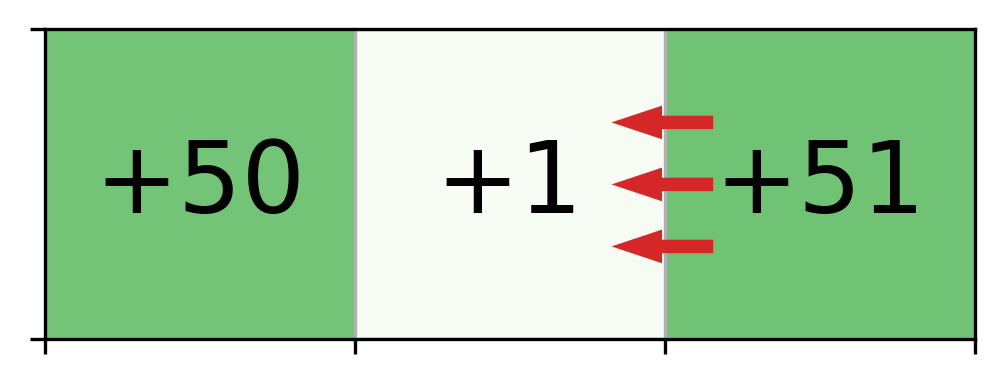}
        \end{subfigure}%
        ~
        \begin{subfigure}[c]{0.6\linewidth}
            \includegraphics[width=\linewidth]{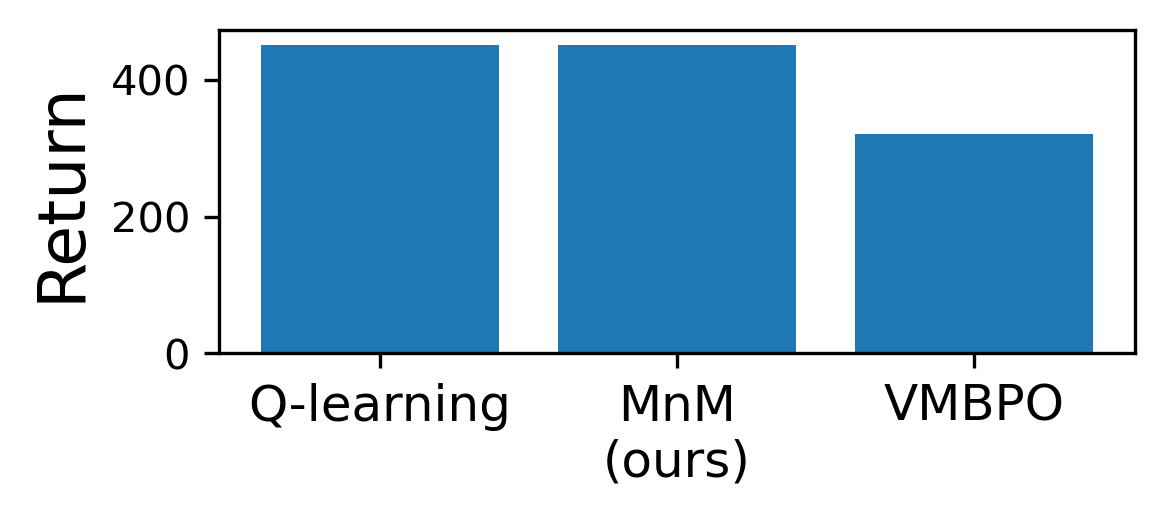}
        \end{subfigure}
        \vspace{-1em}
        \caption{\footnotesize Testing for risk seeking behavior} \label{fig:gridworld-line}
    \end{subfigure}
    \hfill
    \begin{subfigure}[b]{0.4\textwidth}
        \centering
       \includegraphics[width=\linewidth]{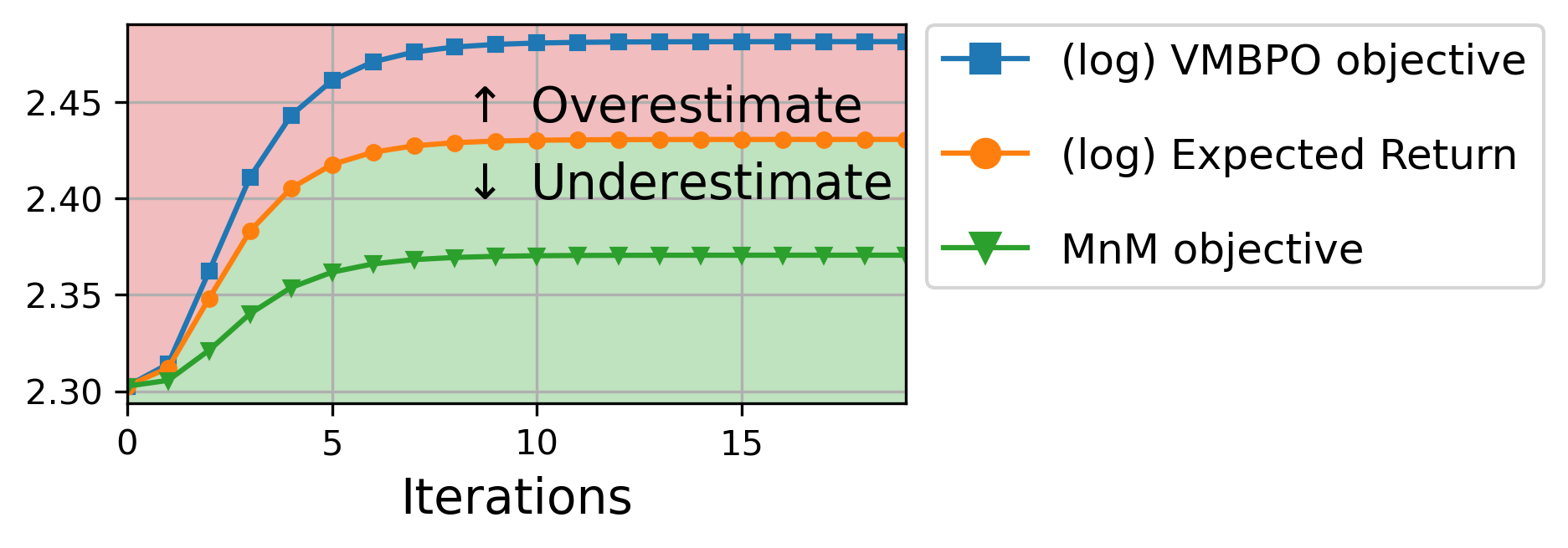}
        \vspace{-1em}
        \caption{\footnotesize Comparing objectives.} \label{fig:lower-bounds} 
    \end{subfigure}
    \caption{\footnotesize \textbf{Analyzing the lower bound.}
    \figleft \, On a 3-state MDP with stochastic transitions in one state (red arrows), MnM learns the reward-maximizing policy while VMBPO learns a strategy with lower rewards and higher variance (as predicted by theory).
    \figright \, We apply value iteration to the gridworld from Fig.~\ref{fig:stochastic-gridworld} to analytically compute various objectives. As predicted by our theory, the MnM objective is a lower bound on the expected return, whereas the VMBPO objective overestimates the expected return.}
\end{figure}

Like MnM, VMBPO includes a classifier term in the reward function but omits the logarithmic transformation, a difference we expect to cause VMBPO to prefer suboptimal, risk-seeking policies. To test this hypothesis, we use the 3-state MDP in Fig.~\ref{fig:gridworld-line} \figright, where numbers indicate the reward at each state. While moving to the right state yields slightly higher rewards, ``wind'' knocks the agent out of this state with probability 50\% so the reward-maximizing strategy is to move to the left state. While MnM learns the reward-maximizing strategy, VMBPO learns a policy that goes to the right state and receives lower returns. %

Finally, we verify Theorem~\ref{lemma:main-result} by comparing the MnM objective to the true expected return. We also compare to the objective from VMBPO, which looks similar to the MnM objective but omits the logarithmic transformation; our theory predicts that the VMBPO objective will therefore be an \emph{upper} bound on the expected return (see Appendix~\ref{appendix:vmbpo}).
We use the gridworld from Fig.~\ref{fig:stochastic-gridworld} and use a version of MnM based on value iteration to avoid approximation error.
Plotting the MnM objective in Fig.~\ref{fig:lower-bounds} \figright, we observe that it is always a lower bound on the (log) expected return, as predicted by our theory. Also as predicted by the theory, the VMBPO objective overestimates the expected return, illustrating the importance of the logarithmic transformation.

Of the oft-cited benefits of model-based RL is that the learned dynamics model can be re-used to solve new tasks. MnM presents a twist on that story, because MnM does not learn the true environment dynamics but rather learns optimistic dynamics (see Sec.~\ref{sec:lower-bound}). In Appendix~\ref{appendix:additional-experiments} (Fig.~\ref{fig:dynamics-transfer}), we examine MnM's effectiveness at transferring dynamics to different tasks in the stochastic gridworld from Fig.~\ref{fig:stochastic-gridworld}. We find that the (optimistic) dynamics learned by MnM do not slow learning on dissimilar tasks, but can accelerate learning of challenging, similar tasks.

\subsection{Comparisons On Higher-Dimensional Tasks}
\label{sec:comparisons}

Our next experiments use continuous-control robotic tasks to study whether the new model-learning objective suggested by our lower bound can be stably applied to higher-dimensional control tasks. While the MnM-approx method tested in this section deviates from our theoretical derivation (see Sec.~\ref{sec:exploration}), these experiments are nonetheless useful for providing preliminary evidence about whether the proposed model objective can be efficiently estimated and stably optimized. 

We use MBPO~\citep{janner2019trust} as a baseline for model-based RL because it achieves state-of-the-art results and is a prototypical example of model-based RL algorithms that use maximum likelihood models. Because our method differs from the baseline (MBPO) along only one dimension now (new model update, same policy update), these experiments will directly test the utility of our proposed model update. Experimental details are in Appendix~\ref{appendix:details}.

We first use three locomotion tasks from the OpenAI Gym benchmark~\citep{brockman2016openai} to compare MnM-approx to MBPO and VMBPO. The VMBPO curves are taken directly from that paper. As shown in Fig.~\ref{fig:mujoco}, MnM-approx performs roughly on par with MBPO and outperforms~\citep{chow2020variational}, a more recent model-based method also addresses the objective mismatch problem by maximizing an \emph{upper} bound on expected returns (1-sided p-values: $p = 0.04, 0.00, 0.03$).

We next use the ROBEL manipulation benchmark~\citep{ahn2020robel} to compare how MnM-approx and MBPO handle tasks with more complicated dynamics.
As shown in Fig.~\ref{fig:dclaw}, MBPO struggles to learn three of the four tasks, likely because the dynamics are hard. In contrast, MnM-approx solves these tasks, likely because the GAN-like model is more accurate. MnM-approx outperforms MBPO on all tasks ($p \le 0.03$) and outperforms the model-free baseline on 2/4 tasks ($p \le 0.02$)

We next compare to prior methods for addressing the objective mismatch problem using the \texttt{DClawScrewFixed-v0} task. We compare to VAML~\citep{farahmand2017value} and the value-weighted maximum likelihood approach proposed in~\citet{lambert2020objective}. As shown in Fig.~\ref{fig:model-ablation}, these alternative model learning objectives perform poorly on this task. %

\begin{figure}[t]
    \centering
    \vspace{-1em}
    \begin{subfigure}[b]{0.45\textwidth}
        \centering
        \includegraphics[width=\linewidth]{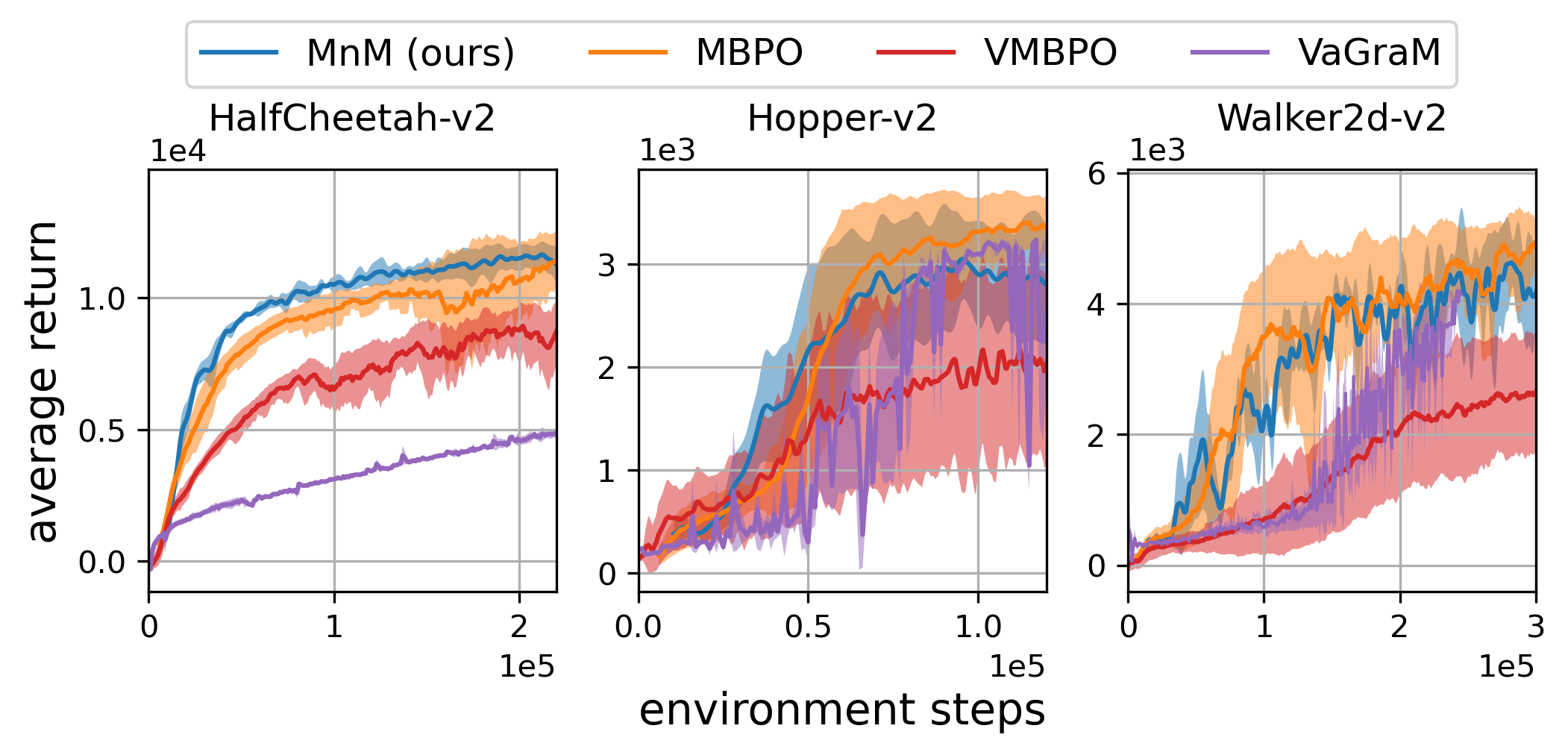}
        \caption{\footnotesize OpenAI Gym benchmark}
        \label{fig:mujoco} 
    \end{subfigure}%
    ~
    \begin{subfigure}[b]{0.55\textwidth}
        \centering
        \includegraphics[width=\linewidth]{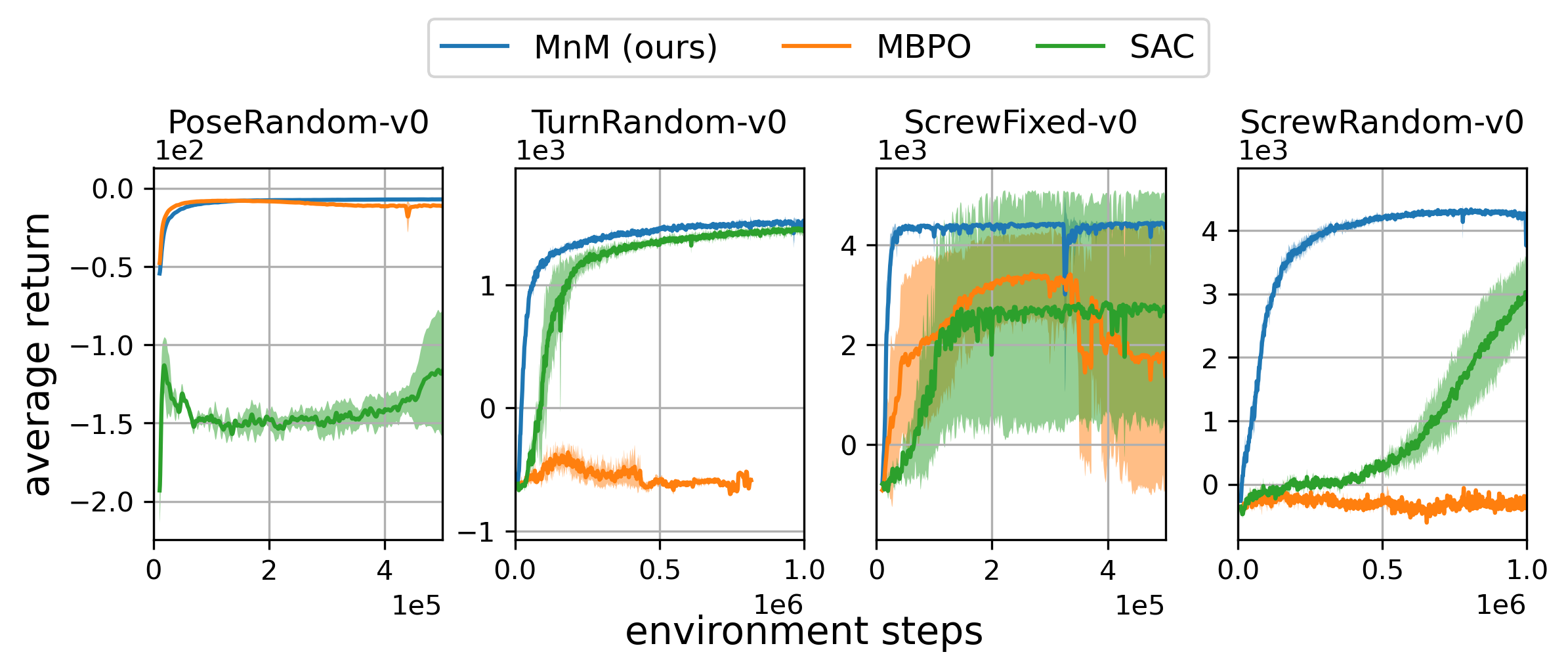}
        \caption{\footnotesize  ROBEL manipulation benchmark} \label{fig:dclaw}  
    \end{subfigure}
    \caption{\footnotesize \textbf{Comparison on two benchmarks}.
    \figleft \, On the OpenAI gym benchmark~\citep{brockman2016openai}, MnM-approx performs on par with a prior state-of-the-art method (MBPO~\citep{janner2019trust}), while consistently outperforming a recent method that addresses objective mismatch (VMBPO~\citep{chow2020variational}).
    \figright \, The ROBEL manipulation benchmark~\citep{ahn2020robel} contains complex contact dynamics that are challenging to model. MBPO performs poorly on these tasks, often worse than model-free SAC~\citep{haarnoja2018sac}. %
    }
\end{figure}

\begin{figure}[t]
    \centering
    \begin{subfigure}[b]{0.3\textwidth}
    \centering
    \includegraphics[width=0.8\linewidth]{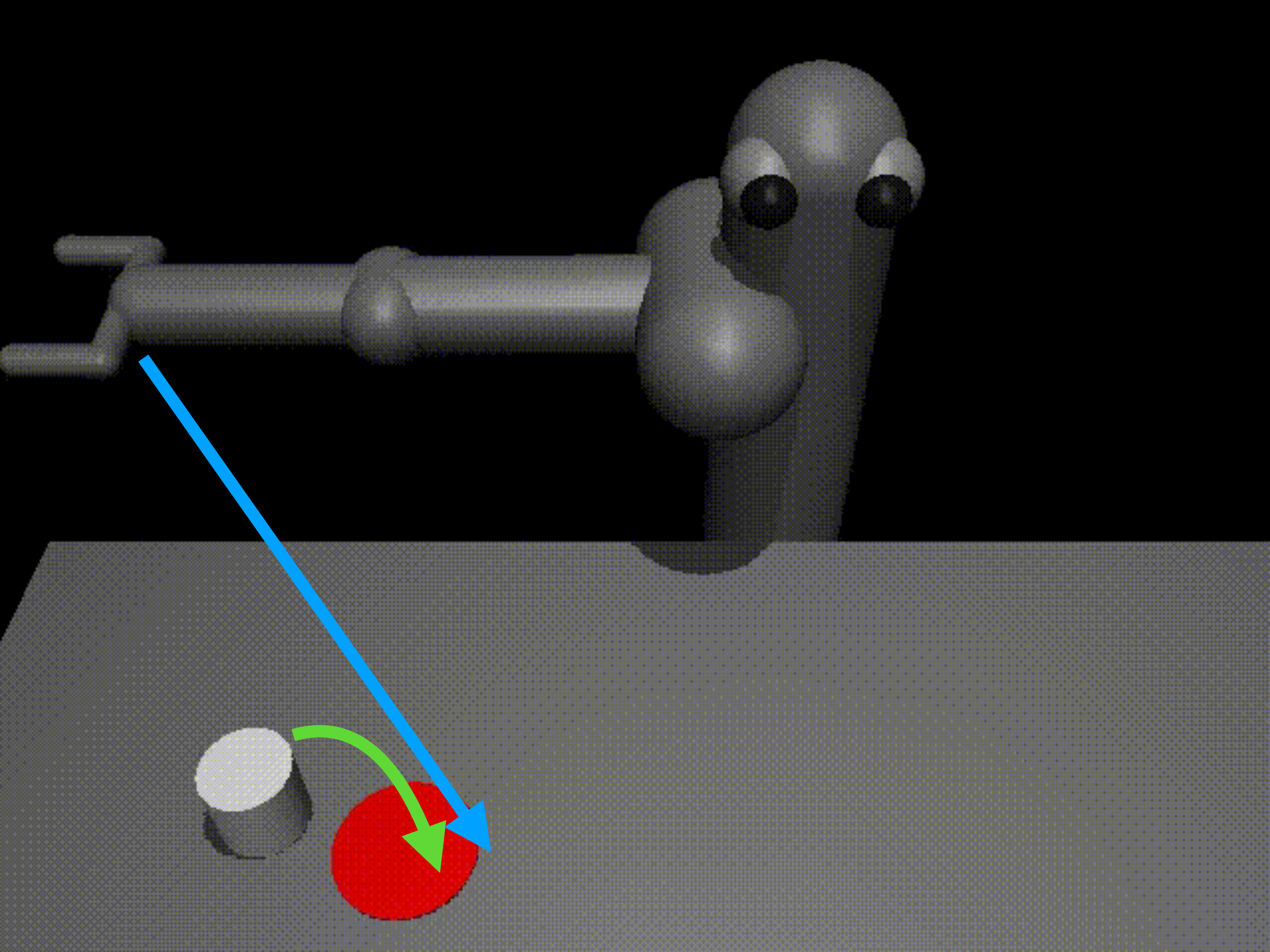}
    \end{subfigure}
    \hspace{1.5em}
    \begin{subfigure}[b]{0.6\textwidth}
    \centering
    \includegraphics[width=0.8\linewidth]{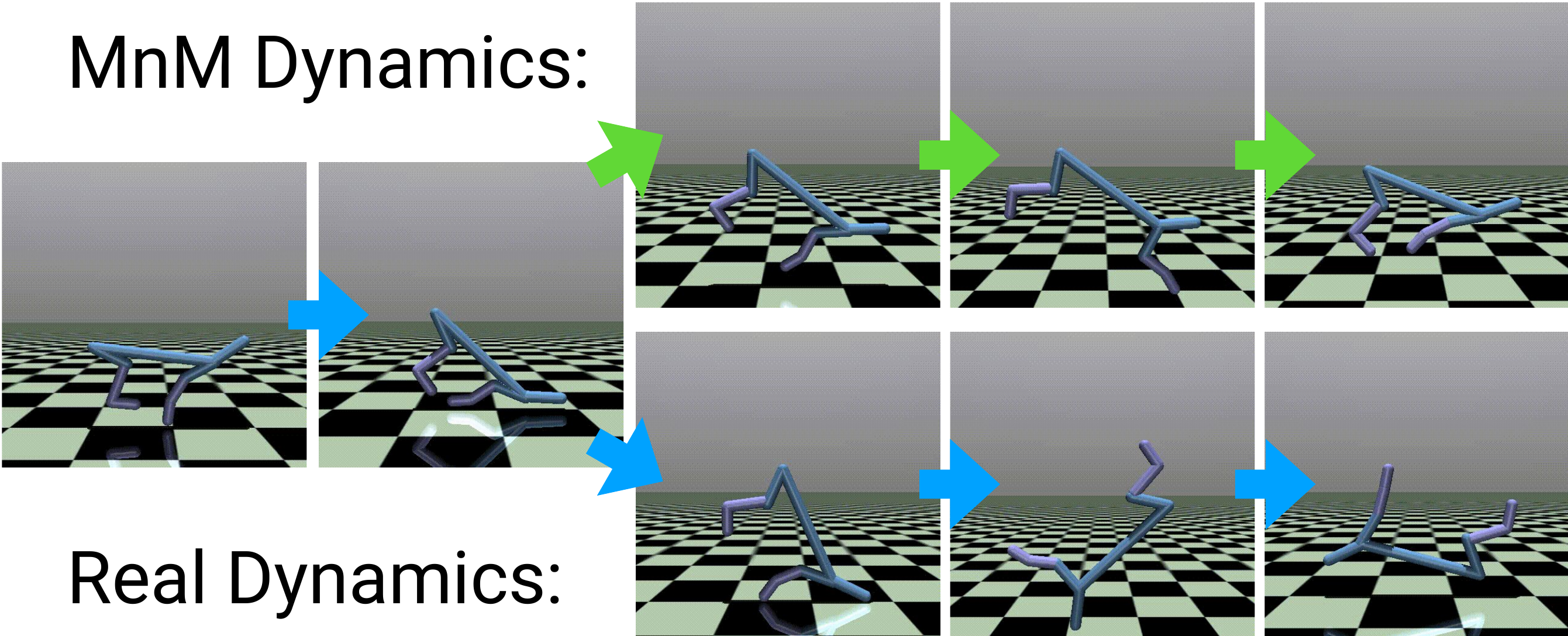}
    \end{subfigure}
    \caption{{\footnotesize \textbf{Optimistic Dynamics}:
    \figleft \,  On the \texttt{Pusher-v2} task, the MnM dynamics model makes the puck move towards the puck move towards the gripper before being grasped.
    \figright \,  On the \texttt{HalfCheetah-v2} task, the MnM dynamics model helps the agent stay upright after tripping.}}
    \label{fig:optimistic-dynamics}
    \vspace{-1.5em}
\end{figure}

\begin{wrapfigure}[14]{R}{0.5\textwidth}
    \centering
    \vspace{-1em}
    \includegraphics[width=\linewidth]{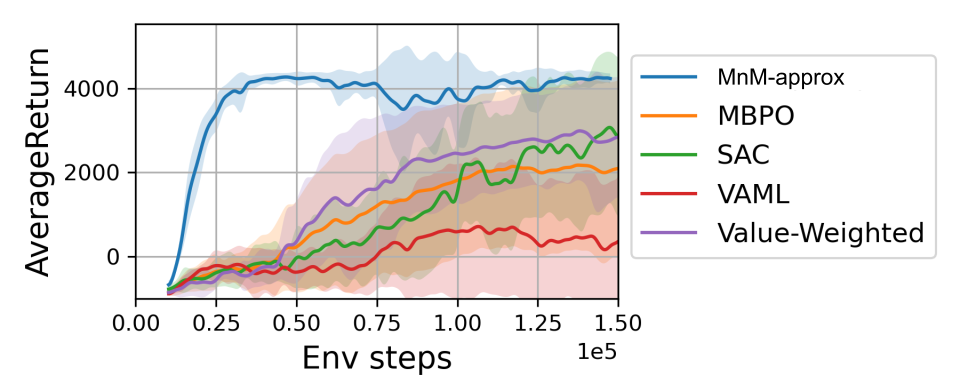}
    \vspace{-1em}
    \caption{\footnotesize\textbf{Alternative model learning objectives}: Using the \texttt{DClawScrewFixed-v0} task, we compare MnM-approx and MBPO~\citep{janner2019trust} to two additional model learning objectives suggested in the literature, VAML~\citep{farahmand2017value} and value-weighted maximum likelihood~\citep{lambert2020objective}. MnM-approx outperforms these alternative approaches.} \label{fig:model-ablation}
\end{wrapfigure}

To better understand why MnM-approx sometimes outperforms the maximum likelihood baseline (MBPO), we visualized the Q-values throughout training. We used \texttt{metaworld-drawer-open-v2}, a task where we found a noticeable difference in the performance between MBPO and MnM-approx.
Fig.~\ref{fig:q-values} shows that MnM-approx yields Q values that are more accurate and more stable than MBPO, perhaps because MBPO learns a policy that exploits inaccuracies in the learned model.

\begin{figure}[t]
            \vspace{-1em}
    \centering
    \begin{subfigure}[b]{0.45\textwidth}
        \centering
        \includegraphics[width=\linewidth]{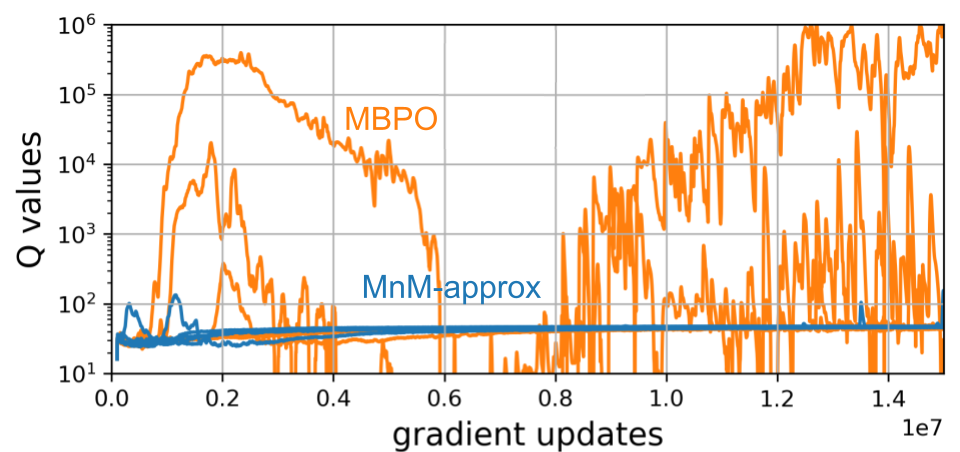}
        \caption{\footnotesize Model exploitation}
        \label{fig:q-values}      
    \end{subfigure}%
    ~
    \begin{subfigure}[b]{0.45\textwidth}
        \centering
        \includegraphics[width=\linewidth]{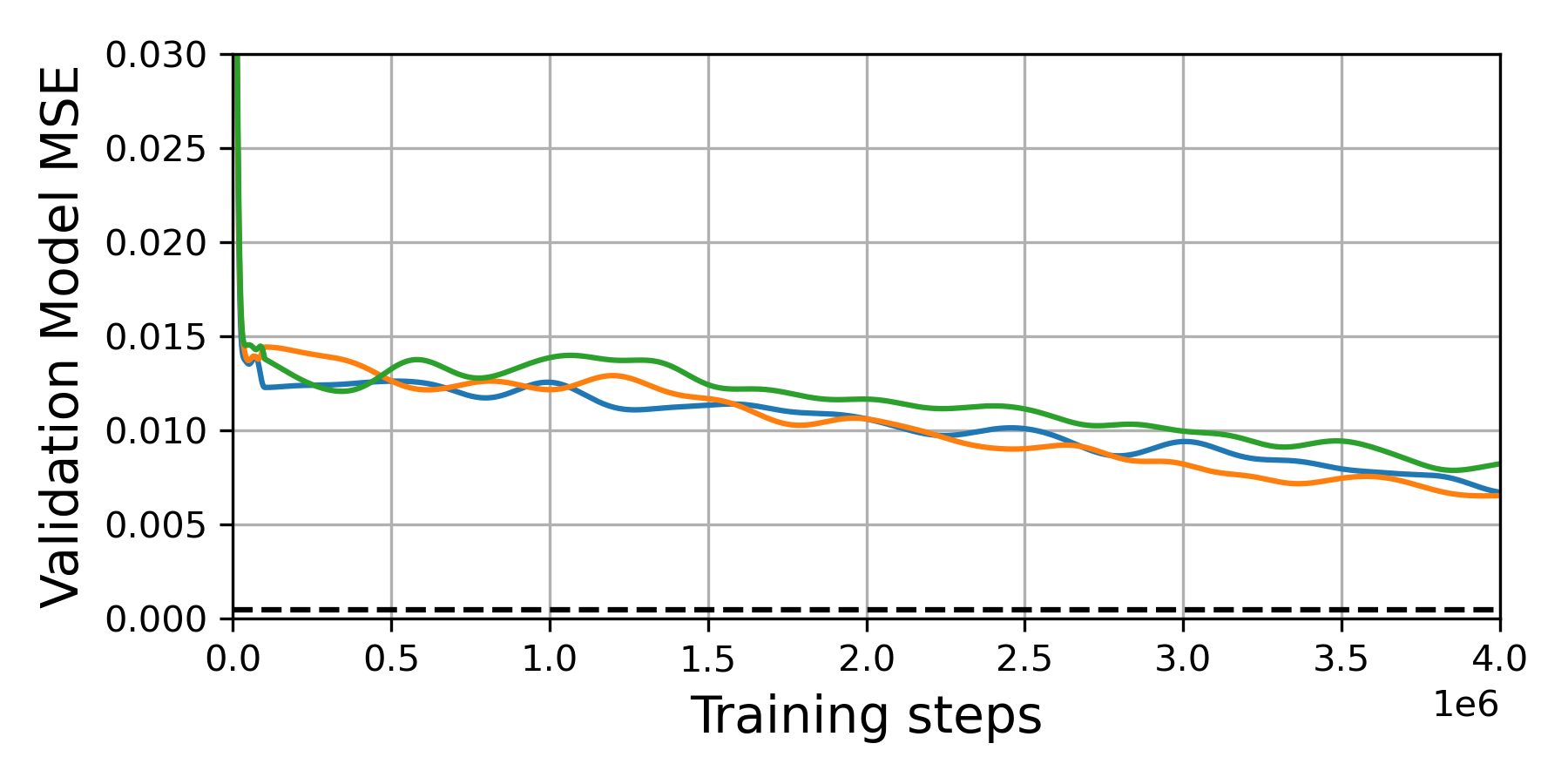}
        \caption{\footnotesize MnM-approx trains stably.}
        \label{fig:gan-stability} 
    \end{subfigure}
    \caption{\footnotesize \textbf{Analyzing MnM-approx.}
            \figleft \; The very large Q values of MBPO suggest model exploitation, which our method appears to avoid.
            \figright \;  Despite resembling a GAN, the MnM-approx dynamics model trains stably. Different colors correspond to different random seeds, and the dashed line corresponds to the minimum validation MSE of an MLE dynamics model.}
        \vspace{-1.em}
\end{figure}

To study the stability of MnM-approx, we plot the validation MSE of the MnM-approx model when training on the \texttt{DClawScrewRandom-v0} task.
As shown in Fig.~\ref{fig:gan-stability}, the MSE decreases stably, indicating that the adversarial nature of the MnM-approx model training does not create instabilities. In Appendix~\ref{appendix:additional-experiments} (Fig.~\ref{fig:model-mse}, ~\ref{fig:q-values-more}) we visualize Q value stability and MSE model error on additional environments. In Appendix~\ref{appendix:additional-experiments} (Fig.~\ref{fig:model-horizon}), we study the effect of model horizon on MnM in a continuous control task, and find that performance degrades linearly as we increase this value.

We visualize the dynamics learned by MnM on two robotic control tasks in Fig~\ref{fig:optimistic-dynamics}. These tasks have deterministic dynamics, so our theory would predict that an idealized version of MnM would learn a dynamics model exactly equal to the deterministic dynamics. However, our implementation relies on function approximation (neural networks) to learn the dynamics, and the limited capacity of function approximators makes otherwise-deterministic dynamics appear stochastic. On the \texttt{Pusher-v2} task, the MnM dynamics cause the puck to move towards the robot arm even before the arm has come in contact with the puck. While this movement is not physically realistic, it may make the exploration problem easier. On the \texttt{HalfCheetah-v2} task, the MnM dynamics increase the probability that the agent remains upright after tripping, likely making it easier for the agent to learn how to run. We expect that the implicit stochasticity caused by function approximation to be especially important for real-world tasks, where the complexity of the real dynamics often dwarfs the capacity of the learned dynamics model.

\section{Conclusion}
\label{sec:conclusion}

The main contribution of this paper is an objective for model-based RL that is both efficient to compute and globally valid. Moreover, the model and policy are jointly optimized using this same objective, thereby addressing the objective mismatch problem. This \emph{joint optimization} will ease and accelerate the design of future model-based RL algorithms.

The main limitation of this paper is the classifier term. Because this term is estimated, the objective used in practice can fail to be a lower bound when the classifier is not Bayes-optimal. Additionally, this classifier term can hinder exploration and degrade performance on continuous control tasks in the online setting. We encourage future investigations into \emph{joint optimization} based approaches that do explicitly address the role of exploration, perhaps including the dataset as an additional optimization variable.

{\footnotesize
\vspace{2em}
\paragraph{Acknowledgements.}
We thank Laura Smith, Marvin Zhang, and Dibya Ghosh for helpful discussions, and thank Anusha Nagabandi and Michael Janner for help with experiments. We thank Yinlam Chow for help sharing the VMBPO baselines and thank Danijar Hafner for reviewing a draft of the paper. This material is supported by the Fannie and John Hertz Foundation and the NSF GRFP (DGE1745016).

}

\clearpage
\appendix

\section{Tightening the lower bound}
\label{appendix:tight}
We now introduce a modification to our lower bound that does make the bound tight. This new lower bound will be more complex than the one introduced above and we have not yet successfully designed an algorithm for maximizing it. Nonetheless, we believe that presenting the bound may prove useful for the design of future model-based RL algorithms.

We will use $\gL_\gamma(\theta)$ to denote this new lower bound. In addition to the policy and dynamics, this bound will also depend on a time-varying discount, $\gamma_\theta(t)$, in place of the typical $\gamma^t$ term.  Similar learned discount factors have been studied in previous work on model-free RL~\citep{rudner2021outcome}. We define this objective as follows:
\begin{align}
    \gL_\gamma(\theta) \triangleq \E_{q^{\pi_\theta}(\tau)} \bigg[\sum_{t=0}^\infty \gamma_\theta(t) \tilde{r}_\gamma(s_t, a_t, s_{t+1}) \bigg], \label{eq:lb2}
\end{align}
where the augmented reward is now defined as
\footnotesize \begin{align*}
    \tilde{r}_\gamma(s_t, &a_t, s_{t+1}) \triangleq  \log r(s_t, a_t) + \log \left( \frac{\gamma^t}{\gamma_\theta(t)} \right)\\
    & + \frac{1 - \Gamma_\theta(t-1)}{\gamma_\theta(t)} \log \left( \frac{p(s_{t+1} \mid s_t, a_t, s_{t-1}, a_{t-1}, \cdots)}{q_\theta(s_{t+1} \mid s_t, a_t, s_{t-1}, a_{t-1}, \cdots)} \right),
\end{align*} \normalsize
and $\Gamma_\theta(t) = \sum_{t'=0}^t \gamma_\theta(t')$ is the CDF of the learned discount function (i.e., $\gamma_\theta(t)$ is a probability distribution over $t$.). This new lower bound, which differs from our main lower bound by the learnable discount factor, does provide a tight bound on the expected return objective:
\begin{lemma} \label{lemma:tight}
Let an arbitrary policy $\pi(a_t \mid s_t)$ be given. The objective $\gL_\gamma(\theta)$ is also a lower bound on the expected return objective, $\log \E_{\pi}\left[\sum_{t=0}^\infty \gamma^t r(s_t, a_t) \right] \ge \gL_\gamma(\theta)$, and this bound becomes tight at optimality:
\footnotesize \begin{equation*}
    \log \E_{\pi}\bigg[\sum_{t=0}^\infty \gamma^t r(s_t, a_t) \bigg] =  \max_{q^\pi(\tau), \gamma_\theta(t)} \gL_\gamma(\theta).
\end{equation*}\normalsize
\end{lemma}
The proof is presented in Appendix~\ref{proof:lemma-tight}.
One important limitation of this result is that the learned dynamics that maximize this lower bound to make the bound tight may be non-Markovian.
Intriguingly, this analysis suggests that using non-Markovian models, such as RNNs and transformers, may accelerate learning on Markovian tasks. This paper does not propose an algorithm for optimizing this more complex lower bound.

\section{Proofs and Additional Analysis}
\label{appendix:analysis}

\subsection{VMBPO Maximizes an Upper Bound on Return}
\label{appendix:vmbpo}
While MnM aims to maximize the (log) of the expected return, VMBPO aims to maximize the expected \emph{exponentiated} return:
\begin{equation*}
    \text{MnM:} \quad \log \E_{\pi}\left[\sum_{t=0}^\infty \gamma^t r(s_t, a_t) \right], \qquad \text{VMBPO:} \quad \log \E_{\pi}\left[e^{\eta \sum_{t=0}^\infty \gamma^t r(s_t, a_t)} \right],
\end{equation*}
where $\eta > 0$ is a temperature term used by VMBPO. Note that maximizing the log of the expected return, as done by MnM, is equivalent to maximizing the expected return, as the function $\log(\cdot)$ is monotone increasing. However, maximizing the log of the expected \emph{exponentiated} return, as done by VMBPO, is not equivalent to maximizing the expected return. Rather, it corresponds to maximizing a sum of the expected return and the \emph{variance} of the return~\citep[Page 272]{mihatsch2002risk}:
\begin{equation*}
    \frac{1}{\eta} \log \E_{\pi}\left[e^{\eta \sum_{t=0}^\infty \gamma^t r(s_t, a_t)} \right] = \E_{\pi}\left[\sum_{t=0}^\infty \gamma^t r(s_t, a_t) \right] + \frac{\eta}{2} \Var_{\pi}\left[\sum_{t=0}^\infty \gamma^t r(s_t, a_t) \right] + \gO(\eta^2).
\end{equation*}
Thus, in environments with stochastic dynamics or rewards (e.g., the didactic example in Fig.~\ref{fig:gridworld-line}), VMBPO will prefer to receive lower returns if the variance of the returns is much higher. We note that the expected \emph{exponentiated} return is an \emph{upper} bound on the expected return:
\begin{equation*}
    \log \E_{\pi}\left[e^{\eta \sum_{t=0}^\infty \gamma^t r(s_t, a_t)} \right] \ge \eta \E_{\pi}\left[\sum_{t=0}^\infty \gamma^t r(s_t, a_t) \right].
\end{equation*}
This statement is a direct application of Jensen's inequality. The bound holds with a strict inequality in almost all MDPs. The one exception is trivial MDPs where all trajectories have exactly the same return. Of course, even a random policy is optimal for these trivial MDPs.

\subsection{Helper Lemmas}
We start by introducing a simple identity that will help handle discount factors in our analysis.
\begin{lemma} \label{lemma:discount-identity}
Define $p(H) = \textsc{Geom}(1 - \gamma)$ as the geometric distribution. Let discount factor $\gamma \in (0, 1)$ and random variable $x_t$ be given. Then the following identity holds:
\begin{equation*}
    \E_{p(H)} \left[\sum_{t=0}^H x_t \right] = \sum_{t=0}^\infty \gamma^t x_t.
\end{equation*}
\end{lemma}
The proof involves substituting the definition of the Geometric distribution and then rearranging terms.
\begin{proof}
\begin{align*}
    \E_{p(H)} \left[\sum_{t=0}^H x_t \right]
    &= (1 - \gamma) \sum_{H=0}^\infty \gamma^H \sum_{t=0}^H x_t \\
    &= (1 - \gamma) \left( x_0 + \gamma (x_0 + x_1) + \gamma^2 (x_0 + x_1 + x_2) + \cdots \right) \\
    &= (1 - \gamma) \left( x_0 (1 + \gamma + \gamma^2 + \cdots) + x_1 (\gamma + \gamma^2 + \cdots) + \cdots \right) \\
    &= (1 - \gamma) \left( x_0 \frac{1}{1 - \gamma} + x_1 \frac{\gamma}{1 - \gamma} + x_2 \frac{\gamma^2}{1 - \gamma} + \cdots \right)\\
    &= \sum_{t=0}^\infty \gamma^t x_t.
\end{align*}
\end{proof}

The second helper lemma describes how the discounted expected return objective can be written as the expected \emph{terminal} reward of a mixture of finite-length episodes.
\begin{lemma} \label{lemma:finite-horizon}
Define $p(H) = \textsc{Geom}(1 - \gamma)$ as the geometric distribution, and $p(\tau \mid H)$ as a distribution over length-$H$ episodes. We can then write the expected discounted return objective as follows:
\begin{align}
    \E_{p(\tau \mid H = \infty)}\left[\sum_{t=0}^\infty \gamma^t r(s_t, a_t) \right]
    &= \frac{1}{1 - \gamma}\E_{p(H)} \left[ \E_{p(\tau \mid H=H)}\left[r(s_H, a_H) \right] \right] \\
    &= \frac{1}{1 - \gamma} \iint p(H) p(\tau \mid H = H) r(s_H, a_H) d \tau dH.
\end{align}
\end{lemma}
\begin{proof}
The first identity follows from the definition of the geometric distribution. The second identity writes the expectations as integrals, which will make future analysis clearer.
\end{proof}

\subsection{Proof of Theorem~\ref{lemma:main-result}}
\label{proof:lemma-main-result}

\begin{proof}
{\scriptsize
\begin{align*}
     \log \E_{\pi}\left[\sum_{t=0}^\infty \gamma^t r(s_t, a_t) \right]
     & \stackrel{(a)}{=} \log \frac{1}{1 - \gamma} \iint p(H) p(\tau \mid H = H) r(s_H, a_H) d \tau dH \\
     & \hspace{-13em} = \log \iint p(H) \frac{p(\tau \mid H = H)}{q_\theta(\tau \mid H = H)}q_\theta(\tau \mid H = H) r(s_H, a_H) d \tau dH - \log (1 - \gamma)\\
     & \hspace{-13em} \stackrel{(b)}{\ge} \int p(H) \left( \log \int \frac{p(\tau \mid H = H)}{q_\theta(\tau \mid H = H)}q_\theta(\tau \mid H = H) r(s_H, a_H) d \tau \right) dH - \log (1 - \gamma) \\
     & \hspace{-13em} \stackrel{(c)}{\ge} \iint p(H) q_\theta(\tau \mid H = H) \left( \log p(\tau \mid H = H) - \log q_\theta (\tau \mid H = H) + \log r(s_H, a_H) \right) d\tau dH - \log (1 - \gamma)\\
     & \hspace{-13em} \stackrel{(d)}{=} \iint p(H) q_\theta(\tau \mid H = H) \left( \left(\sum_{t=0}^H \log p(s_{t+1} \mid s_t, a_t) + \cancel{\log \pi_\theta(a_t \mid s_t)} - \log q_\theta(s_{t+1} \mid s_t, a_t) - \cancel{\log \pi_\theta(a_t \mid s_t)} \right) + \log r(s_H, a_H) \right) d\tau dH  \\
     & \hspace{-10em} - \log (1 - \gamma)\\
     & \hspace{-13em} \stackrel{(e)}{=} \iint p(H) q_\theta(\tau \mid H = \infty)\left( \left(\sum_{t=0}^H \log p(s_{t+1} \mid s_t, a_t) - \log q_\theta(s_{t+1} \mid s_t, a_t) \right) + \log r(s_H, a_H) \right) d\tau dH - \log (1 - \gamma)\\
     & \hspace{-13em} \stackrel{(f)}{=} \int q_\theta(\tau) \int p(H) \left( \left(\sum_{t=0}^H \log p(s_{t+1} \mid s_t, a_t) - \log q_\theta(s_{t+1} \mid s_t, a_t) \right) + \log r(s_H, a_H) \right) dH d\tau - \log (1 - \gamma)\\
     & \hspace{-13em} \stackrel{(g)}{=} \int q_\theta(\tau) \E_{p(H)} \left[ \left( \sum_{t=0}^H \log p(s_{t+1} \mid s_t, a_t) - \log q_\theta(s_{t+1} \mid s_t, a_t) \right) + \log r(s_H, a_H) \right] d\tau - \log (1 - \gamma)\\
     & \hspace{-13em} \stackrel{(h)}{=} \int q_\theta(\tau) \sum_{t=0}^\infty \gamma^t \left(\log p(s_{t+1} \mid s_t, a_t) - \log q_\theta(s_{t+1} \mid s_t, a_t) + (1 - \gamma) \log r(s_H, a_H) \right) d\tau - \log (1 - \gamma) \\
     & \hspace{-13em} \stackrel{(i)}{=} \E_{q_\theta(\tau)} \left[ \sum_{t=0}^\infty \gamma^t \left(\log p(s_{t+1} \mid s_t, a_t) - \log q_\theta(s_{t+1} \mid s_t, a_t) + (1 - \gamma) \log r(s_H, a_H) - (1 - \gamma)\log (1 - \gamma) \right) \right].
\end{align*}}
For \emph{(a)}, we applied Lemma~\ref{lemma:finite-horizon}.
For \emph{(b)}, we applied Jensen's inequality.
For \emph{(c)}, we applied Jensen's inequality again.
For \emph{(d)}, we substituted the definitions of $p_\theta(\tau \mid H)$ and $q_\theta(\tau \mid H)$.
For \emph{(e)}, we noted that the term inside the summation only depends on the first $H$ steps of the trajectory, so collecting longer trajectories will not change the result. This allows us to rewrite the integral as an expectation using a single infinite-length trajectory.
For \emph{(f)}, we recalled the definition $q_\theta(\tau) = q_\theta(\tau = H = \infty)$ and swap the order of integration.
For \emph{(g)}, we express the inner integral over $p(H)$ as an expectation.
For \emph{(h)}, we applied the identity from Lemma~\ref{lemma:discount-identity}.
For \emph{(i)}, we moved the constant $\log (1 - \gamma)$ back inside the integral and rewrote the integral as an expectation. We have thus obtained the desired result.
\end{proof}

\clearpage

\subsection{Proof of Lemma~\ref{lemma:tight}}
\label{proof:lemma-tight}
Before presenting the proof of Theorem~\ref{lemma:main-result} itself, we show how we derived the lower bound in this more general case. While this step is not required for the proof, we include it because it sheds light on how similar lower bounds might be derived for other problems. We define $\gamma_\theta(H)$ to be a learned distribution over horizons $H$. We then proceed, following many of the same steps as for the proof of Theorem~\ref{lemma:main-result}.
{\scriptsize
\begin{align}
     \log \E_{\pi}\left[\sum_{t=0}^\infty \gamma^t r(s_t, a_t) \right] \nonumber
     & \stackrel{(a)}{=} \log \iint \frac{p(\tau, H)}{q_\theta(\tau, H)}q_\theta(\tau, H) r(s_H, a_H) d \tau dH - \log (1 - \gamma) \nonumber \\
     & \hspace{-13em} \stackrel{(b)}{\ge} \iint q_\theta(\tau, H) \left( \log p(\tau, H) - \log q_\theta(\tau, H) + \log r(s_H, a_H) d \tau \right) dH - \log (1 - \gamma) \label{eq:l-gamma-traj} \\
     & \hspace{-13em} \stackrel{(c)}{=} \int \sum_{H=0}^\infty \gamma_\theta(H) q_\theta(\tau \mid H) \left( \left( \sum_{t=0}^H \log p(s_{t+1} \mid s_t, a_t) - \log q_\theta(s_{t+1} \mid s_t, a_t) \right) + \log p(H) - \log \gamma_\theta(H) + \log r(s_H, a_H) d \tau \right) - \log (1 - \gamma) \nonumber \\
     & \hspace{-13em} \stackrel{(d)}{=} \int q_\theta(\tau \mid H=\infty) \sum_{H=0}^\infty \gamma_\theta(H) \left( \left( \sum_{t=0}^H \log p(s_{t+1} \mid s_t, a_t) - \log q_\theta(s_{t+1} \mid s_t, a_t) \right) + \log p(H) - \log \gamma_\theta(H) + \log r(s_H, a_H) d \tau \right) - \log (1 - \gamma) \nonumber \\
     & \hspace{-13em} \stackrel{(e)}{=} \int q_\theta(\tau) \sum_{H=0}^\infty \gamma_\theta(H) \left( \left( \sum_{t=0}^H \log p(s_{t+1} \mid s_t, a_t) - \log q_\theta(s_{t+1} \mid s_t, a_t) \right) + \log p(H) - \log \gamma_\theta(H) + \log r(s_H, a_H) d \tau \right) - \log (1 - \gamma) \nonumber \\
     & \hspace{-13em} \stackrel{(f)}{=} \E_{q_\theta(\tau)} \left[ \sum_{H=0}^\infty \gamma_\theta(H) \left( \left( \sum_{t=0}^H \log p(s_{t+1} \mid s_t, a_t) - \log q_\theta(s_{t+1} \mid s_t, a_t) \right) + \cancel{\log (1 - \gamma)} + H \log \gamma - \log \gamma_\theta(H) + \log r(s_H, a_H) \right) \right]  - \cancel{\log (1 - \gamma)} \nonumber \\
     & \hspace{-13em} \stackrel{(g)}{=} \E_{q_\theta(\tau)} \left[ \sum_{H=0}^\infty \left(\sum_{t=H}^\infty q(t) \right) \left(\log p(s_{H+1} \mid s_H, a_H) - \log q_\theta(s_{H+1} \mid s_H, a_H) \right) + \gamma_\theta(H) \left(H \log \gamma - \log \gamma_\theta(H) + \log r(s_H, a_H) \right) \right] \nonumber \\
     & \hspace{-13em} \stackrel{(h)}{=} \E_{q_\theta(\tau)} \left[ \sum_{H=0}^\infty \left(1 - \sum_{t=0}^{H-1} q(t) \right) \left(\log p(s_{H+1} \mid s_H, a_H) - \log q_\theta(s_{H+1} \mid s_H, a_H) \right) + \gamma_\theta(H) \left(H \log \gamma - \log \gamma_\theta(H) + \log r(s_H, a_H) \right) \right] \nonumber \\
     & \hspace{-13em} \stackrel{(i)}{=} \E_{q_\theta(\tau)} \left[ \sum_{H=0}^\infty \left(1 - \Gamma_\theta(H-1) \right) \left(\log p(s_{H+1} \mid s_H, a_H) - \log q_\theta(s_{H+1} \mid s_H, a_H) \right) + \gamma_\theta(H) \left(H \log \gamma - \log \gamma_\theta(H) + \log r(s_H, a_H) \right) \right] \nonumber \\
     & \hspace{-13em} \stackrel{(j)}{=} \E_{q_\theta(\tau)} \left[ \sum_{H=0}^\infty \gamma_\theta(H) \left( \frac{1 - \Gamma_\theta(H-1)}{\gamma_\theta(H)} \left(\log p(s_{H+1} \mid s_H, a_H) - \log q_\theta(s_{H+1} \mid s_H, a_H) \right) + H \log \gamma - \log \gamma_\theta(H) + \log r(s_H, a_H) \right) \right]. \nonumber
\end{align}}
For \emph{(a)}, we applied Lemma~\ref{lemma:finite-horizon} and multiplied the integrand by $\frac{q_\theta(\tau \mid H = H)\gamma_\theta(H)}{q_\theta(\tau \mid H = H)\gamma_\theta(H)} = 1.$
For \emph{(b)}, we applied Jensen's inequality.
For \emph{(c)}, we factored $p(\tau, H) = p(\tau, H) p(H)$ and $q_\theta(\tau, H) = q(\tau \mid H) \gamma_\theta(H)$. Note that under the joint distribution $p(\tau, H)$, the horizon $H \sim p(H) = \textsc{Geom}(1 - \gamma)$ is independent of the trajectory, $\tau$.
For \emph{(d)}, we rewrote the expectation as an expectation over a single infinite-length trajectory and simplified the summand.
For \emph{(e)}, we recall the definition $q_\theta(\tau) = q_\theta(\tau = H = \infty)$.
For \emph{(f)}, we rewrote the integral as an expectation and wrote out the definition of the geometric distribution, $p(H)$.
For \emph{(g)}, we regrouped the difference of dynamics terms.
For \emph{(h)}, we noted used the fact that $\sum_{t=0}^{H-1} \gamma_\theta(t) + \gamma_{t=H}^\infty \gamma_\theta(t) = 1$.
For \emph{(i)}, we substituted the definition of the CDF function.
For \emph{(j)}, we rearranged terms so that all were multiplied by the discount $\gamma_\theta(H)$.
Thus, we have obtained the desired result. We now prove Lemma~\ref{lemma:tight}, showing that Eq.~\ref{eq:lb2} becomes tight at optimality.
\begin{proof}
{\scriptsize
\begin{align}
    \gL_\gamma(\theta) & \stackrel{(a)}{=} \iint q_\theta(\tau, H) \left( \log p(\tau, H) - \log q_\theta(\tau, H) + \log r(s_H, a_H) d \tau \right) dH - \log (1 - \gamma) \nonumber \\
    & \qquad \stackrel{(b)}{=} \iint q_\theta(\tau) \gamma_\theta(H \mid \tau) \left( \log p(\tau) + \log p(H) - \log q_\theta(\tau) - \log \gamma_\theta(H \mid \tau) + \log r(s_H, a_H) d \tau \right) dH - \log (1 - \gamma) \label{eq:tight-2}
\end{align}}
For \emph{(a)}, we undo some of the simplifications above, going back to Eq.~\ref{eq:l-gamma-traj}
For \emph{(b)}, we factor $q_\theta(\tau, H) = q_\theta(\tau) \gamma_\theta(H \mid \tau)$ and $p(\tau, H) = p(\tau) p(H)$.
At this point, we can solve analytically for the optimal discount distribution, $\gamma_\theta(H \mid \tau)$:
\begin{equation}
    \gamma_\theta^*(H \mid \tau) = \frac{p(H) r(s_H, a_H)}{\sum_{H'=0}^\infty p(H') r(s_{H'}, a_{H'})} = \frac{p(H) r(s_H, a_H)}{(1 - \gamma)R(\tau)} \label{eq:q-h-opt}
\end{equation}
In the second equality, we substitute the definition of $R(\tau)$.
We then substitute Eq.~\ref{eq:q-h-opt} into our expression for $\gL_\gamma(\theta)$ and simplify the resulting expression.
{\scriptsize
\begin{align*}
    \hspace{-10em} \gL_\gamma(\theta)
    &= \iint q_\theta(\tau) \gamma_\theta(H \mid \tau) \left( \log p(\tau) + \cancel{\log p(H)} - \log q_\theta(\tau) - \cancel{\log p(H)} - \cancel{\log r(s_H, a_H)} + \cancel{\log (1 - \gamma)} + \log R(\tau) + \cancel{\log r(s_H, a_H)} d \tau \right) dH - \cancel{\log (1 - \gamma)} \\
    &= \iint q_\theta(\tau) \gamma_\theta(H \mid \tau) \left( \log p(\tau) - \log q_\theta(\tau) + \log R(\tau) d \tau \right) dH \\
    &= \int q_\theta(\tau) \left( \log p(\tau) - \log q_\theta(\tau) + \log R(\tau) \right) d \tau.
\end{align*}}
In the final line we have removed the integral over $H$ because none of the integrands depend on $H$.
At this point, we can solve analytically for the optimal trajectory distribution, $q_\theta(\tau)$:
\begin{equation}
    q^*(\tau) = \frac{p(\tau) R(\tau)}{\int p(\tau') R(\tau') d \tau'}. \label{eq:q-tau-opt}
\end{equation}
We then substitute Eq.~\ref{eq:q-tau-opt} into our expression for $\gL_\gamma(\theta)$, and simplify the resulting expression:
{\scriptsize
\begin{align*}
    \hspace{-10em} \gL_\gamma(\theta)
    &= \int q_\theta(\tau) \left( \cancel{\log p(\tau)} - \cancel{\log p(\tau)} - \cancel{\log R(\tau)} + \log \int p(\tau') R(\tau') d \tau' + \cancel{\log R(\tau)} \right) d \tau \\
    &= \log \int p(\tau) R(\tau) d \tau = \log \E_{\pi}\left[\sum_{t=0}^\infty \gamma^t r(s_t, a_t) \right].
\end{align*}}

We have thus shown that the lower bound $\gL_\gamma$ becomes tight when we use the optimal distribution over trajectories $q_\theta(\tau)$ and optimal learned discount $\gamma_\theta(H \mid \tau)$.
\end{proof}

\subsection{A lower bound for goal-reaching tasks.}
\label{appendix:goals}

Many RL problems can be better formulated as goal-reaching problems, a formulation that does not require defining a reward function. We now introduce a variant of our method for goal-reaching tasks. Using $\rho^\pi(s_{t+})$ to denote the discounted state occupancy measure of policy $\pi$, we define the goal-reaching objective as maximizing the probability density of reaching a desired goal $s_g$:
\begin{equation}
    \max_\theta \log \rho^{\pi_\theta}(s_{t+} = s_g).
\end{equation}
We refer the reader to~\citet{eysenbach2020c} for a more detailed discussion of this objective. For simplicity, we assume that the goal is fixed, noting that the multi-task setting can be handled by conditioning the policy on the commanded goal. Similar to Theorem~\ref{lemma:main-result}, we can construct a lower bound on the goal-conditioned RL problem:
\begin{lemma} \label{lemma:goals}
Let initial state distribution $p_1(s_1)$, real dynamics $p(s_{t+1} \mid s_t, a_t)$, reward function $r(s_t, a_t) > 0$, discount factor $\gamma \in (0, 1)$, and goal $g$ be given. Then the following bound holds for \emph{any} dynamics $q(s_{t+1} \mid s_t, a_t)$ and policy $\pi(a_t \mid s_t)$:
\begin{align}
    \log p^{\pi_\theta}(s_{t+} = s_g) \ge \E_{q^{\pi_\theta}(\tau)} \left[\sum_{t=0}^\infty \gamma^t \tilde{r}(s_t, a_t) \right],
\end{align}
where
\begin{align*}
\tilde{r}_g(s_t, a_t, s_{t+1}) \triangleq & (1 - \gamma) (\log p(s_{t+1} = s_g \mid s_t, a_t) - \log q(s_{t+1} = s_g \mid s_t, a_t) - \log(1 - \gamma)) \\
&+ \log p(s_{t+1} \mid s_t, a_t) - \log q(s_{t+1} \mid s_t, a_t).
\end{align*}
\end{lemma}
The proof, presented below, is similar to the proof of Theorem~\ref{lemma:main-result}.
The first term in the reward function, the log ratio of reaching the commanded goal one time step in the future, is similar to prior work~\citep{rudner2021outcome}. The correction term $\log p - \log q$ incentivizes the policy to avoid transitions where the model is inaccurate, and can be estimated using a separate classifier.
One important aspect of this goal-reaching problem is that it is entirely data-driven, avoiding the need for any manually-designed reward functions.

\begin{proof}

{\scriptsize
\begin{align*}
    \log \rho^{\pi_\theta}(s_{t+} = s_g)
     & = \log \iint p(H) p(\tau \mid H = H) p(s_g \mid s_H, a_H) d \tau dH \\
     & \hspace{-5em} = \log \iint p(H) \frac{p(\tau \mid H = H)}{q_\theta(\tau \mid H = H)}q_\theta(\tau \mid H = H) \frac{p(s_g \mid s_H, a_H)}{q_\theta(s_g \mid s_H, a_H)}q_\theta(s_g \mid s_H, a_H) d \tau dH - \log (1 - \gamma)\\
     & \hspace{-5em} \ge \iint p(H) q_\theta(\tau \mid H = H) \left( \log p(\tau \mid H = H) - \log q_\theta (\tau \mid H = H) + \log p(s_g \mid s_H, a_H) - \log q_\theta(s_g \mid s_H, a_H) \right) d\tau dH - \log (1 - \gamma)\\
     & \hspace{-5em} = \iint p(H) q_\theta(\tau \mid H = \infty) \left(\sum_{t=0}^H \log p(s_{t+1} \mid s_t, a_t) - \log q_\theta(s_{t+1} \mid s_t, a_t) \right) + \log p(s_g \mid s_H, a_H) - \log q_\theta(s_g \mid s_H, a_H) d\tau dH - \log (1 - \gamma)\\
     & \hspace{-5em} = \int q_\theta(\tau) \int p(H) \left(\sum_{t=0}^H \log p(s_{t+1} \mid s_t, a_t) - \log q_\theta(s_{t+1} \mid s_t, a_t)\right) + \log p(s_g \mid s_H, a_H) - \log q_\theta(s_g \mid s_H, a_H) dH d\tau - \log (1 - \gamma)\\
     & \hspace{-5em} = \int q_\theta(\tau) \sum_{t=0}^\infty \gamma^t \left( \log p(s_{t+1} \mid s_t, a_t) - \log q_\theta(s_{t+1} \mid s_t, a_t) + (1 - \gamma)(\log p(s_g \mid s_t, a_t) - \log q_\theta(s_g \mid s_t, a_t) \right) d\tau - \log (1 - \gamma))\\
     & \hspace{-5em} = \E_{q_\theta(\tau)} \left[ \sum_{t=0}^\infty \gamma^t \left( \log p(s_{t+1} \mid s_t, a_t) - \log q_\theta(s_{t+1} \mid s_t, a_t) + (1 - \gamma)(\log p(s_g \mid s_t, a_t) - \log q_\theta(s_g \mid s_t, a_t) - \log (1 - \gamma) \right) \right].
\end{align*}}

\end{proof}
Similar to the more complex lower bound presented in Eq.~\ref{eq:lb2}, this lower bound on goal-reaching can be modified (by learning a discount factor) to become a tight lower bound. The resulting objective would resemble a model-based version of the algorithm from~\citet{rudner2021outcome}.

\subsection{Derivation of Model Objective (Eq.~\ref{eq:model-objective})}
\label{appendix:single-transition}

Our lower bound depends on entirely trajectories sampled from the learned dynamics. In this section, we show how the same objective can be expressed as an expectation of transitions. This expression is easier to optimize, as it does not require backpropagating gradients through time.
We start by writing our lower bound, conditioned on a current state $s_t$.
\begin{align*}
    \E_{\substack{\pi(a_t \mid s_t),\\q_\theta(s_{t+1 \mid s_t, a_t)}}}& \left[\sum_{t' = t}^\infty \gamma^{t' - t} \tilde{r}(s_{t'}, a_{t'}) \mid s_t\right] \\
    &\stackrel{}{=} \E_{\substack{\pi(a_t \mid s_t),\\q_\theta(s_{t+1 \mid s_t, a_t)}}}\left[\tilde{r}(s_t, a_t, s_{t+1}) + \gamma V(s_{t+1}) \mid s_t\right] \\
    &\stackrel{(a)}{=} \E_{\substack{\pi(a_t \mid s_t),\\q_\theta(s_{t+1 \mid s_t, a_t)}}}\left[(1 - \gamma) \log r(s_t, a_t) + \log \frac{C_\phi(s_t, a_t, s_{t+1})}{1 - C_\phi(s_t, a_t, s_{t+1})} - (1 - \gamma)\log(1 - \gamma) + \gamma V(s_{t+1}) \mid s_t\right] \\
\end{align*}
In \emph{(a)}, we substituted the definition of the augmented return. For the purpose of optimizing the dynamics model, we can ignore all terms that do not depend on $s_{t+1}$. Removing these terms, we arrive at our model training objective (Eq.~\ref{eq:model-objective})

\section{Additional Experiments}
\label{appendix:additional-experiments}

\paragraph{Fig.~\ref{fig:model-ablation}.} We compare MnM to a number of alternative model learning methods. MBPO~\citep{janner2019trust} uses a standard maximum likelihood model. We implement a version of VAML~\citep{farahmand2017value}, which augments the maximum likelihood loss with an additional temporal difference loss; the model should predict next states that have low Bellman error. Finally, we compare to a variant of the MBPO maximum likelihood model that weights transitions based on the Q values, an idea discussed (but not actually implemented) in~\citet{lambert2020objective}. We implement this value weighting method by computing the Q values for the current states and computing a softmax over the batch dimension to obtain per-example weights.

\begin{figure}[t]
    \vspace{-1em}
    \centering
    \begin{subfigure}[b]{0.45\textwidth}
        \centering
        \includegraphics[width=\linewidth]{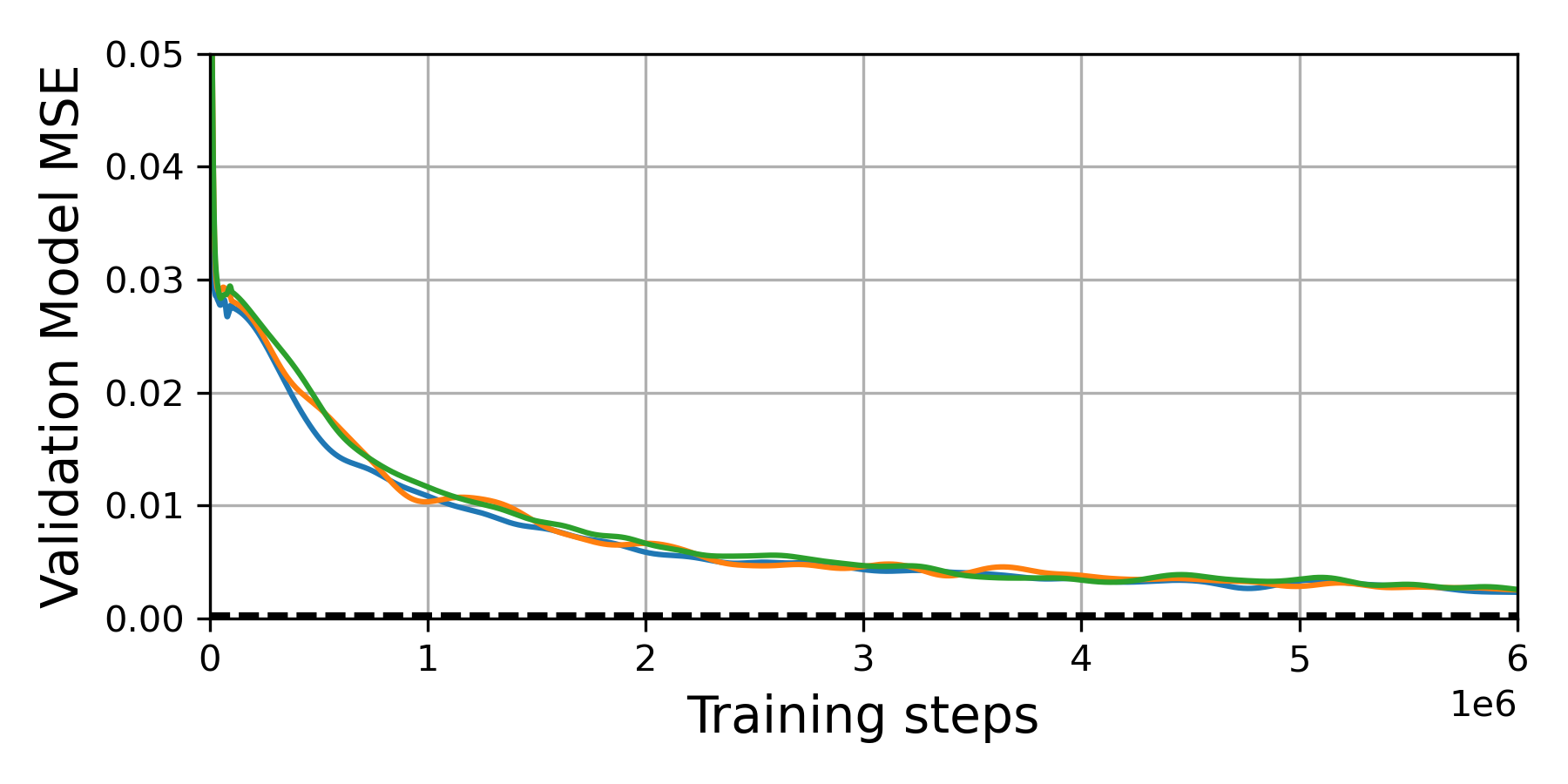}
        \caption{\footnotesize \texttt{DClawScrewFixed-v0}}
    \end{subfigure}%
    ~
    \begin{subfigure}[b]{0.45\textwidth}
        \centering
        \includegraphics[width=\linewidth]{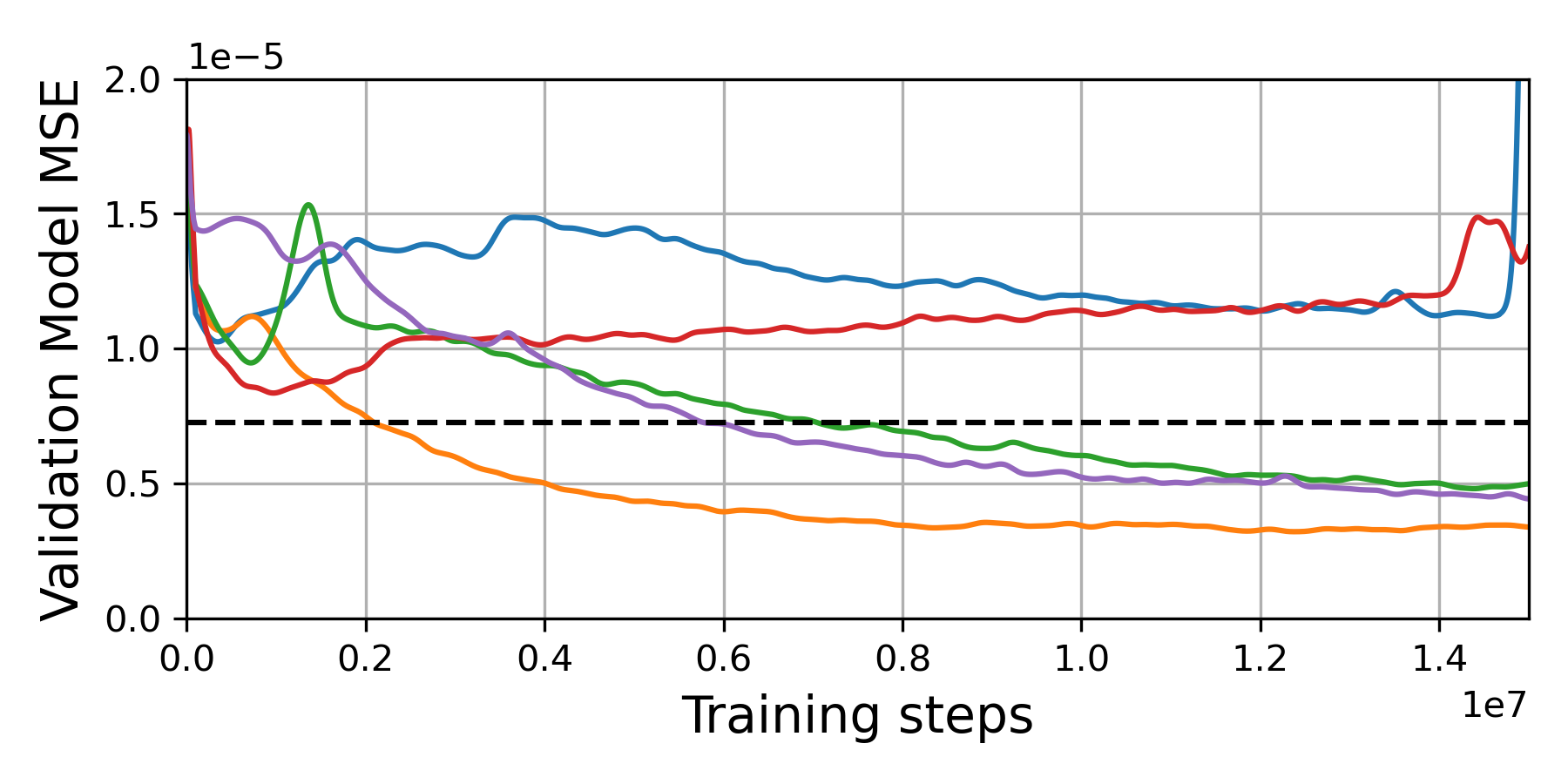}
        \caption{\footnotesize \texttt{metaworld-drawer-open-v2}}
    \end{subfigure}
    \caption{\footnotesize \textbf{MnM-approx Model MSE}. Different lines show different random seeds, while the dashed horizontal line shows the minimum MSE of a maximum likelihood model (averaged across seeds). \label{fig:model-mse}}
\end{figure}

\begin{figure}[t]
    \vspace{-1em}
    \centering
    \begin{subfigure}[b]{0.45\textwidth}
        \centering
        \includegraphics[width=\linewidth]{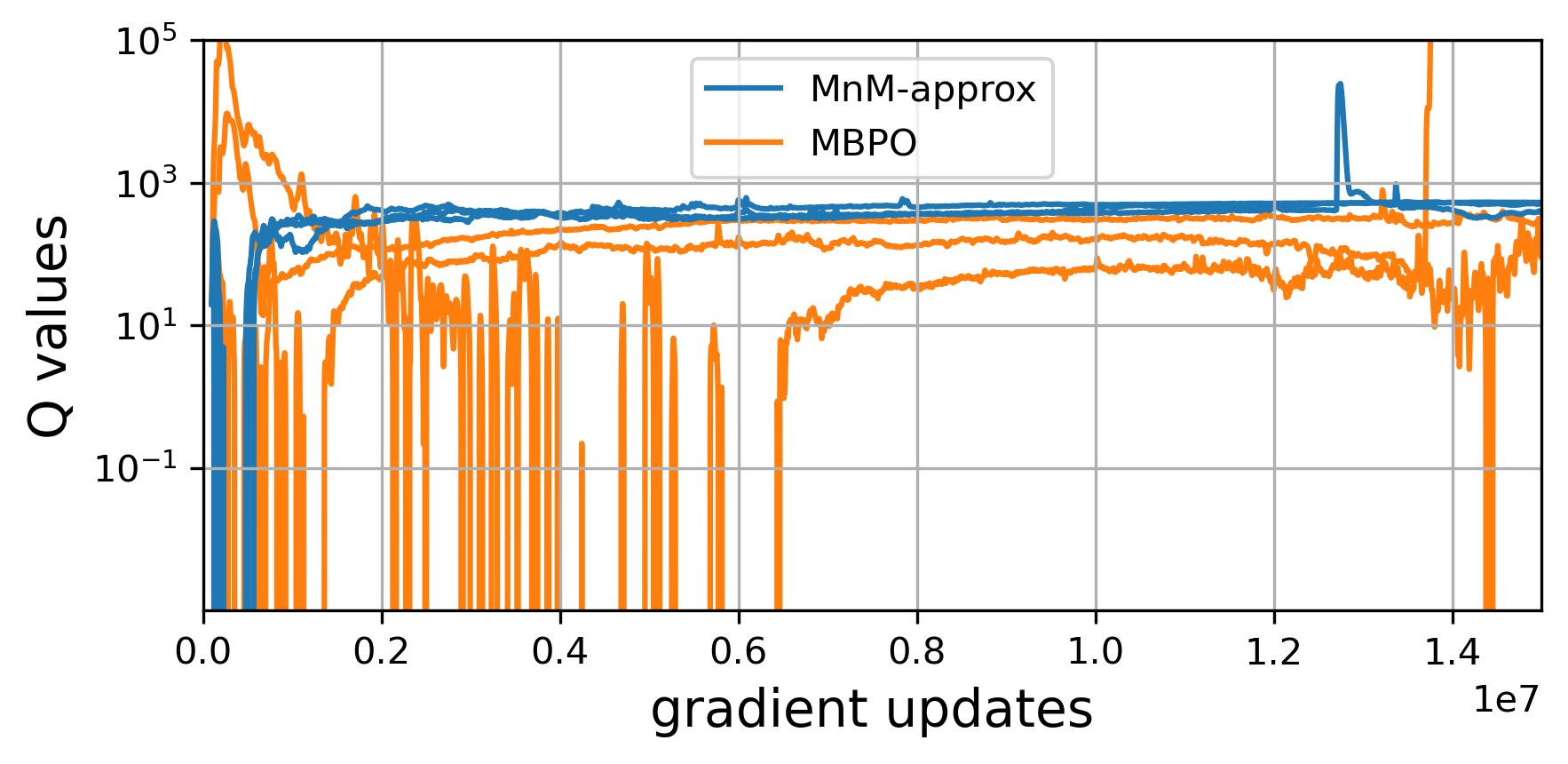}
        \caption{\footnotesize \texttt{DClawScrewFixed-v0}}
    \end{subfigure}%
    ~
    \begin{subfigure}[b]{0.45\textwidth}
        \centering
        \includegraphics[width=\linewidth]{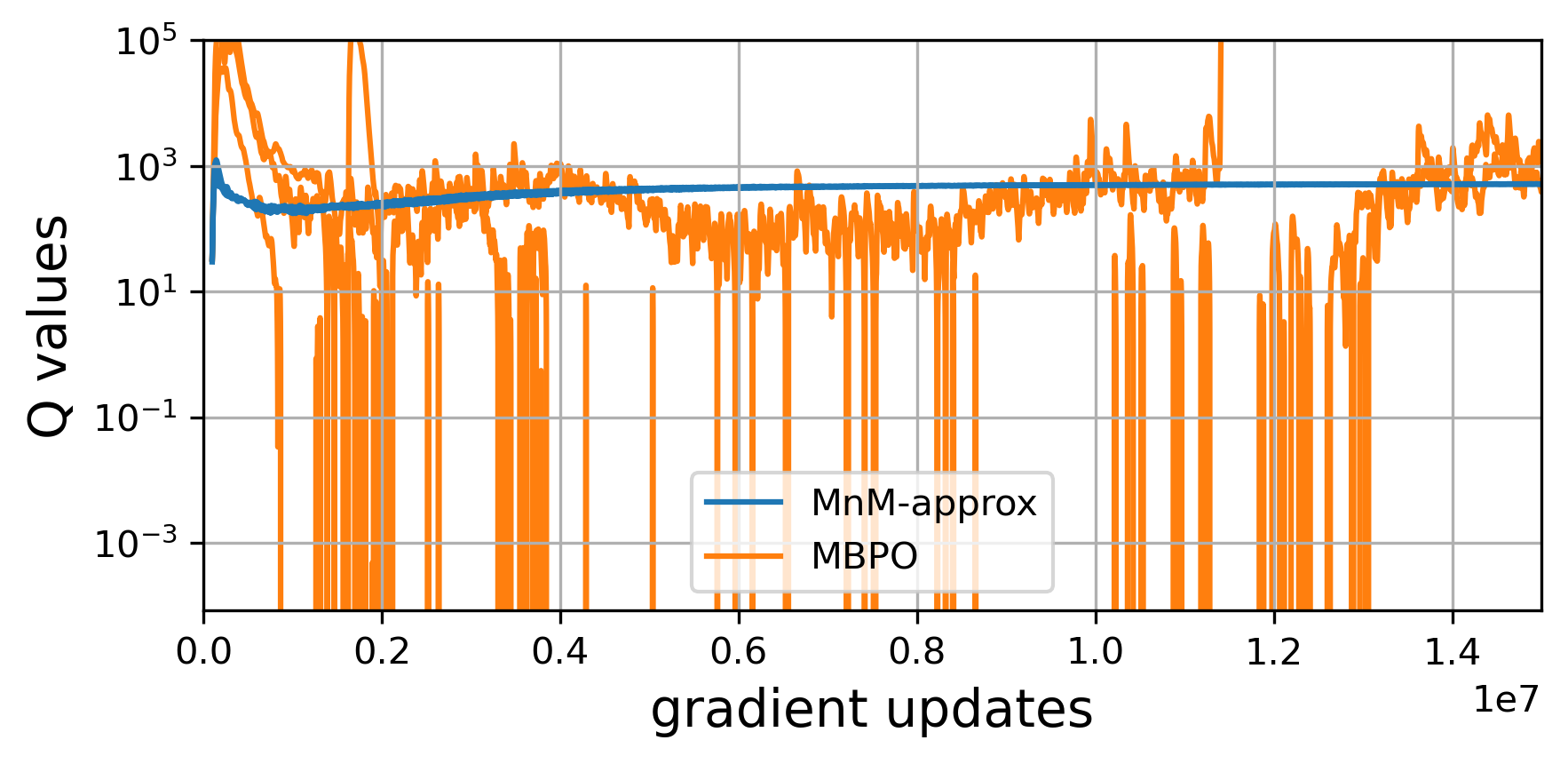}
        \caption{\footnotesize \texttt{DClawScrewRandom-v0}}
    \end{subfigure}
    \caption{\footnotesize \textbf{MnM-approx Q-values} \label{fig:q-values-more}}
\end{figure}

\paragraph{Fig.~\ref{fig:model-mse}} In this experiment, we show the model MSE throughout the training of MnM-approx, similar to Fig.~\ref{fig:gan-stability}. While MnM-approx does not optimize for MSE, the MSE is nonetheless a rough barometer for whether the model is training stably. On the \texttt{DClawScrewFixed-v0} task, we observe that $\frac{3}{3}$ seeds all train stably, converges towards (though not quite reaching) the MSE of the maximum likelihood model. On the \texttt{metaworld-drawer-open-v2} task, the results are a bit more complicated, with $\frac{3}{5}$ seeds achieving a model MSE much lower than the maximum likelihood model, while two seeds seem to have failed to converge. Overall, these plots indicate that MnM-approx often trains stably, but some seeds can fail to converge on some environments.

\begin{wrapfigure}[12]{R}{0.4\textwidth}
    \centering
    \includegraphics[width=\linewidth]{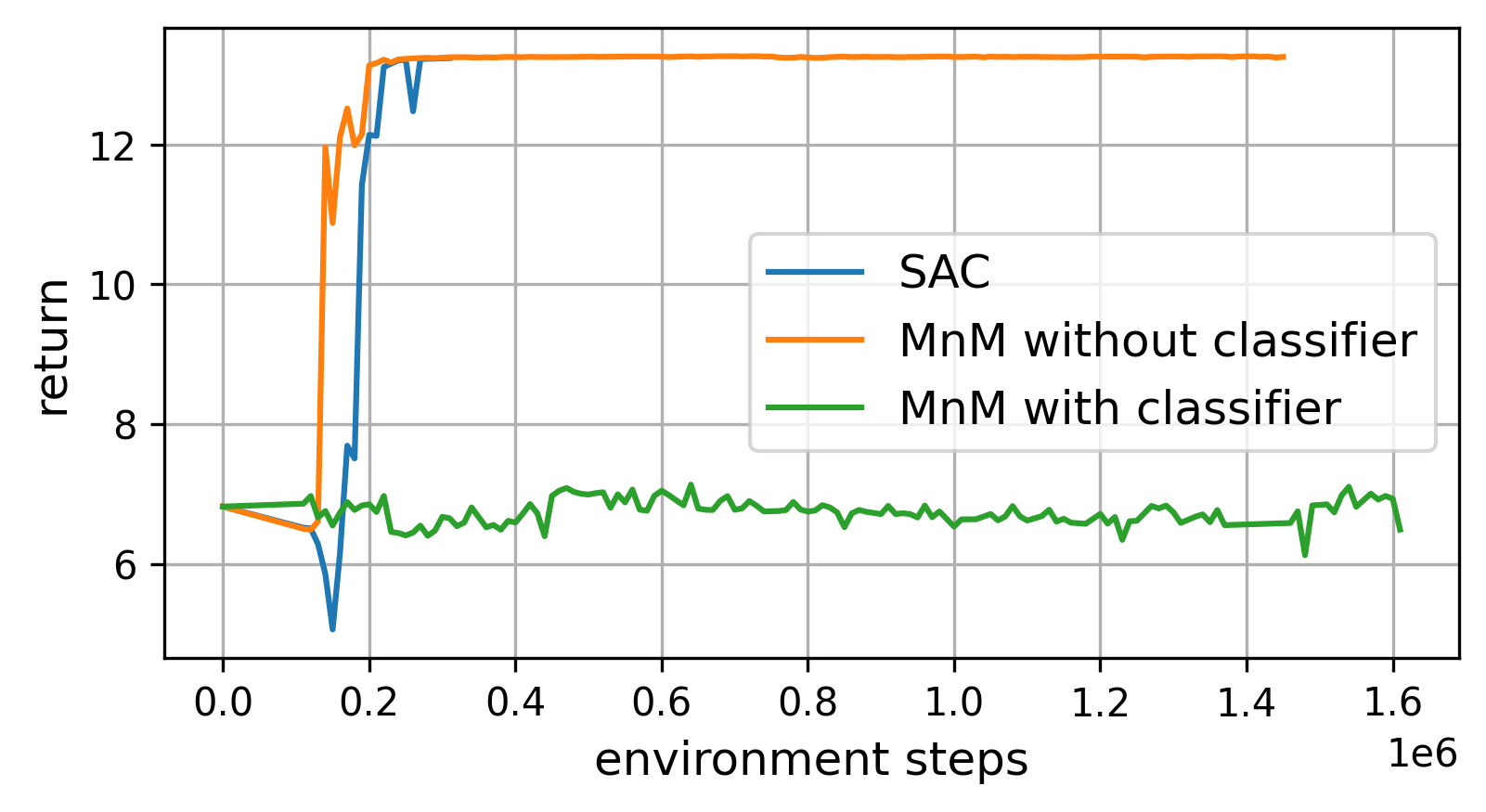}
    \vspace{-1.5em}
    \caption{Adding the classifier term to MnM degrades performance on the \texttt{metaworld-drawer-open-v2} task. \label{fig:no-classifier}}
\end{wrapfigure}

\paragraph{Fig.~\ref{fig:q-values-more}}
In this experiment, we show the Q-values throughout the training of MnM-approx, similar to Fig.~\ref{fig:q-values}. In both environments, we observe that the Q-values from MnM-approx are more stable than the Q-values from MBPO. Note how the Q-values from MBPO tend to peak early during training, suggesting that they have overestimated the true returns and then correct towards a less inflated estimate of the agent's expected return

\paragraph{Fig.~\ref{fig:no-classifier}.} This is a plot from a preliminary version of MnM, when we were testing the effect of the classifier term. In experiments like these, we found that the classifier term significantly hurt performance, motivating us to not include the classifier term in the ``MnM-Approx'' method used in the continuous control experiments.

\begin{wrapfigure}{R}{0.4\textwidth}
    \centering
    \vspace{-1.5em}
    \includegraphics[width=\linewidth]{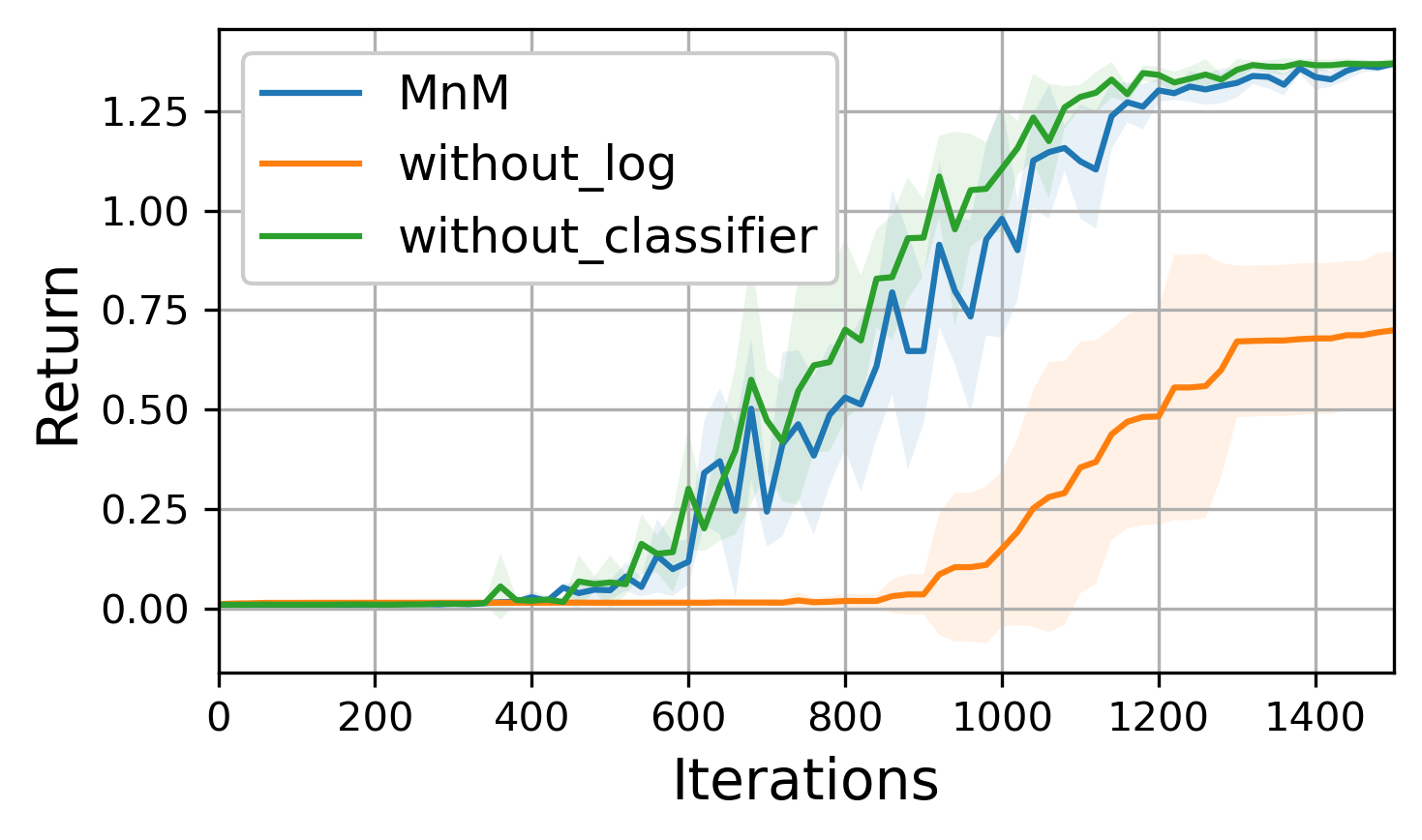}
    \vspace{-1.5em}
    \caption{Stochastic gridworld ablation. \label{fig:stochastic-gridworld-ablation}}
\end{wrapfigure}

\paragraph{Fig.~\ref{fig:stochastic-gridworld-ablation}} In this experiment, we ablate the two factors of the MnM augmented reward (Eq.~\ref{eq:aug-reward}, using the stochastic gridworld from Fig.~\ref{fig:stochastic-gridworld}. The first ablation removes the logarithm, while the second ablation removes the classifier. The results, shown in Fig.~\ref{fig:stochastic-gridworld-ablation}, indicate that the logarithm term is crucial for getting good performance, but that removing classifier term has a relative small effect, and may even boost performance by a small margin. Not to read into these results too much, the aliasing experiment in Fig.~\ref{fig:aliasing} has already demonstrated that the classifier term is critical for ensuring good performance in the presence of function approximation.

\begin{wrapfigure}[11]{R}{0.4\textwidth}
    \centering
    \vspace{-1em}
    \includegraphics[width=\linewidth]{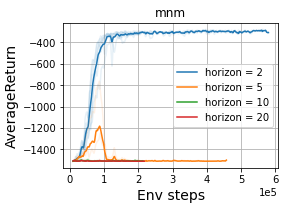}
    \vspace{-2em}
    \caption{\footnotesize\textbf{Model horizon}\label{fig:model-horizon}}
\end{wrapfigure}

\paragraph{Fig.~\ref{fig:model-horizon}}
In our experiments, we used very short model rollouts, only collecting a single transition from the model. Fig.~\ref{fig:model-horizon} shows an ablation experiment on the \texttt{metaworld-drawer-open-v2} task where we increased the model horizon, finding that it uniformly degrades performance.

\begin{figure}[t]
    \vspace{-1em}
    \centering
    \begin{subfigure}[b]{0.25\textwidth}
        \centering
        \includegraphics[width=\linewidth]{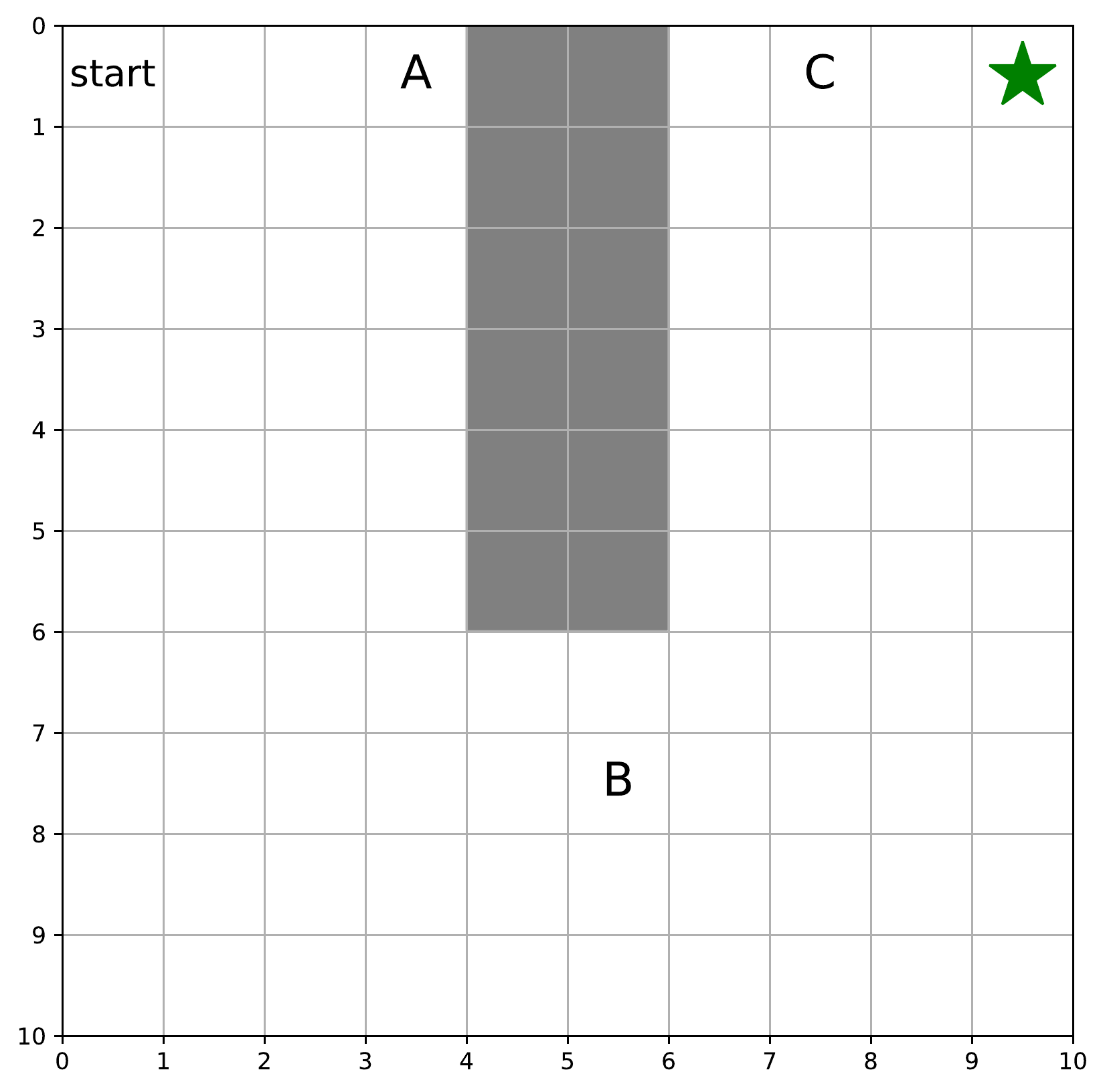}
    \end{subfigure}%
    ~
    \begin{subfigure}[b]{0.75\textwidth}
        \centering
        \includegraphics[width=\linewidth]{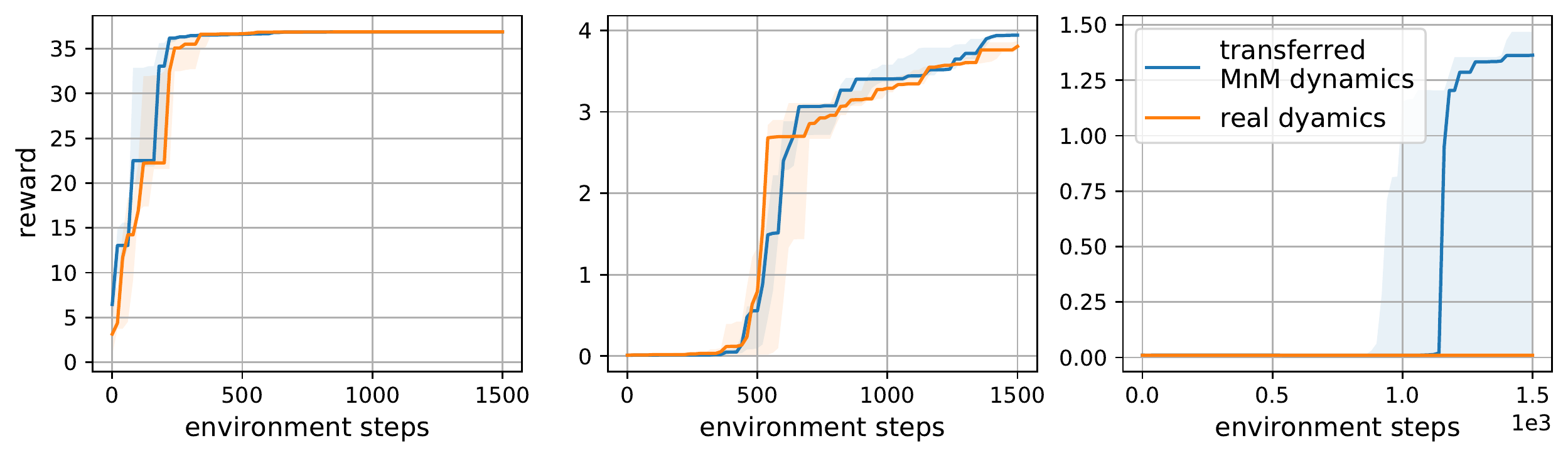}
    \end{subfigure}
    \caption{\footnotesize \textbf{Transferring the learned dynamics to a new task.} We applied Q-learning with either the true environment dynamics or with the dynamics learned by MnM on the original task (reaching the green star). The shaded region corresponds to the $[25\%, 75\%]$ region across 5 random seeds.
    \label{fig:dynamics-transfer}}
\end{figure}

\paragraph{Fig.~\ref{fig:dynamics-transfer}}

In this experiment, we examine the transferability of MnM's learned dynamics model to new tasks. When learning a similar task, we might expect that the optimistic dynamics produced by MnM would be more useful than the true environment dynamics. However, when learning dissimilar tasks, the optimism for one task might hinder exploration towards solving the new task, and might be worse than using the true environment dynamics.

We tested these hypotheses using the stochastic gridworld from Fig.~\ref{fig:stochastic-gridworld}. We took the dynamics learned by MnM (Fig.~\ref{fig:stochastic-gridworld} (right)) and used it to solve new tasks, defined by placing the goal in different locations. As shown in Fig.~\ref{fig:dynamics-transfer} (right), tasks B and C are similar to the original task, in that they involve navigating around the wall; task A is not as similar to the original task, but is much easier.

To study the effectiveness of transferring the dynamics, we applied Q-learning to each of these tasks, either using the (incorrect) optimistic dynamics learned from the original task, or by using the true environment dynamics. We show the results in Fig.~\ref{fig:dynamics-transfer}. On Tasks A and B, we observe little difference between Q-learning with the transferred dynamics model versus the true dynamics model. However, on task C, only by using the transferred dynamics is Q-learning able to solve this challenging task. In summary, these results suggest that the bias of the learned dynamics does not seem to hurt when learning dissimilar tasks, but can accelerate learning of challenging, similar tasks.

\section{Implementation Details}
\label{appendix:details}

All experiments were run on at least three random seeds. Although we were limited by computational constraints, we find that most of our conclusions hold with $p < 0.05$, as noted in the main text.

\subsection{Algorithms}

\paragraph{Value Iteration (Fig.~\ref{fig:stochastic-gridworld} \emph{(right)}, Fig.~\ref{fig:aliasing}, Fig.~\ref{fig:lower-bounds})}
For the tabular experiments that perform value iteration, we perform Polyak averaging of the policy and learned dynamics model with parameter $\tau = 0.5$. We found that the value iteration version of MnM diverged without this Polyak averaging step. Experiments were stopped when the MnM dynamics model (which depends on the value function) changed by less than 1e-6 across iterations, as measured using an $L_0$ norm. For VMBPO we used $\eta = 1$.

\paragraph{Q-learning (Fig.~\ref{fig:stochastic-gridworld})}
For the experiments with Q-learning (both with and without the MnM components), we performed $\epsilon$-greedy exploration with $\epsilon = 0.5$. We used a learning rate of $1e-2$. For this task alone, we compute the MnM dynamics analytically by combining the true environment dynamics with the learned value function, allowing for clearer theoretical analysis. For fair comparison, all methods receive the same amount of data, perform the same number of updates, and are evaluated using the real environment dynamics. For VMBPO we used $\eta = 1$.

\paragraph{SAC for continuous control tasks.}
We used the SAC implementation from TF-Agents~\citep{TFAgents} with the default hyperparameters.

\paragraph{MBPO for continuous control tasks.}
We implement MBPO on top of the SAC implementation from TF-Agents~\citep{TFAgents}. Unless otherwise mentioned, we take the default parameters from this implementation.  We use an ensemble of 5 dynamics models, each with 4 hidden layers of size 256. The dynamics model predicts the whitened difference between the next state and the current state. That is, to obtain the prediction for the next state, the predictions are scaled by a per-coordinate variance, shifted by a per-coordinate mean, and then added to the current state. These whitened predictions are clipped to have a minimum standard deviation of 1e-5; without this, we found that the MBPO model resulted in numerical instability. The model is trained using the standard maximum likelihood objective, with all members of the ensemble being trained on the same data. To sample data from this model we perform 1-step rollouts, starting at states visited in the true dynamics. We perform one batch of rollouts in parallel using a batch size of 256.
To sample the corresponding action, with probability 50\% we take the action that was executed in the true dynamics; with probability 50\% we sample an action from the current policy. We found that this modification slightly improves the results of MBPO. We use a batch size of 256. We have two replay buffers: the model replay buffer has size 256e3 and the replay buffer of real experience has size 1e6. At the start of training, we collect 1e4 transitions from the real environment, train the dynamics model on this experience for 1e5 batches, and only then start training the policy. We use a learning rate of 3e-4 for all components. To stabilize learning, we maintain a target dynamics model using an exponential moving average ($\tau = 0.001$), and use this target dynamics model to sample transitions for training. We update the model, policy, and value functions at the same rate we sample experience from the learned model, which is more frequently than we collect experience from the real environment (see Table~\ref{tab:grad-steps}).

\begin{table}[!t]
    \centering
    \caption{{\footnotesize \textbf{Gradient updates per real environment step}: This parameter was separately tuned for each method and each environment.}}
    \begin{tabular}{c|c|c|c}
     & SAC & MBPO & MnM \\ \hline
        \texttt{HalfCheetah-v2} & - & 40 & 20 \\
        \texttt{Hopper-v2} & - & 40 & 20 \\
        \texttt{Walker2d-v2} & - & 40 & 20\\
        \texttt{metaworld-drawer-open-v2} & 40& 40& 40\\
        \texttt{DClawPoseRandom-v0} & 20& 20& 20\\
        \texttt{DClawTurnRandom-v0} & 40& 40& 40\\
        \texttt{DClawScrewFixed-v0} & 40& 40& 40\\
        \texttt{DClawScrewRandom-v0} & 40& 40& 40\\
    \end{tabular}
    \label{tab:grad-steps}
\end{table}

\paragraph{MnM for Continuous Control Tasks}

We implement MnM on top of the SAC implementation from TF-Agents~\citep{TFAgents}. Unless otherwise mentioned, we take the default parameters from this implementation.
Our model architecture is exactly the same as our MBPO implementation, and we follow the same training protocol.

Unlike MBPO, MnM also learns a classifier for distinguishing real versus model transitions. The classifier architecture is a 2 layers neural network with 1024 hidden units in each layer. We found that this large capacity was important for stable learning. We add input noise with $\sigma = 0.1$ while training the classifier. We whiten the inputs to the classifier by subtracting a coordinate-wise mean and dividing by a coordinate-wise standard deviation. When training the classifier, we take samples from both the dynamics model and the target dynamics model as negative examples, finding that this stabilizes learning somewhat. Following the suggestion of prior work~\citep{salimans2016improved}, we use one-sided label smoothing with value 0.1, only smoothing the negative predictions and not the positive predictions.

We found that gradient penalties and spectral normalization decreased performance. We found that automatically tuning the classifier input noise also decreased performance. We found that mixup had little effect. We found that the loss would often plateau around 1e4 batches, but would eventually start decreasing again after 2e4 - 2e5 batches.

Like the MBPO model, we first collect 1e4 transitions of experience from the real environment using a random policy, then train the dynamics model and classifier for 1e5 batches, and only then start updating the policy. Because the Q values are poor at the start of training, we only add the value term to the model loss (resulting in the optimistic dynamics model) after 2e5 batches (1e5 batches of model training, then 1e5 batches of model+policy training). To further improve stability, we compute the value term in the model loss by taking the minimum over two target value functions (like TD3~\citep{fujimoto2018td3}). We update the model, classifier policy, and value functions at the same rate we sample experience from the learned model, which is more frequently than we collect experience from the real environment (see Table~\ref{tab:grad-steps}).

\subsection{Environments}
This section provides details for the environments used in our experiments. We visualize the environments in Fig.~\ref{fig:envs-subset}.

\begin{figure}[t]
    \centering
    \begin{center}
    \includegraphics[height=1.8cm]{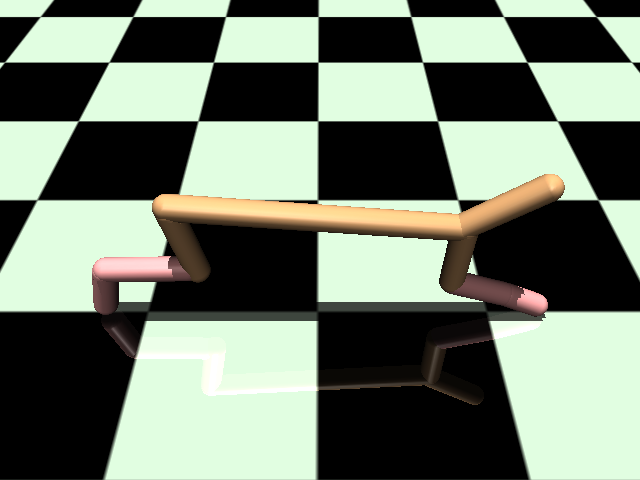}
    \includegraphics[height=1.8cm]{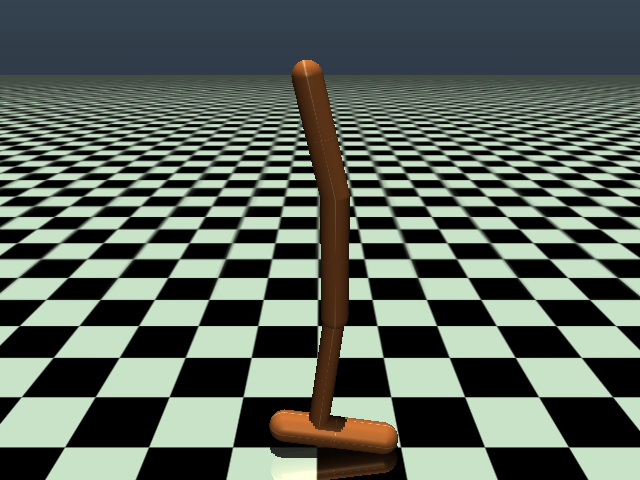}
    \includegraphics[height=1.8cm]{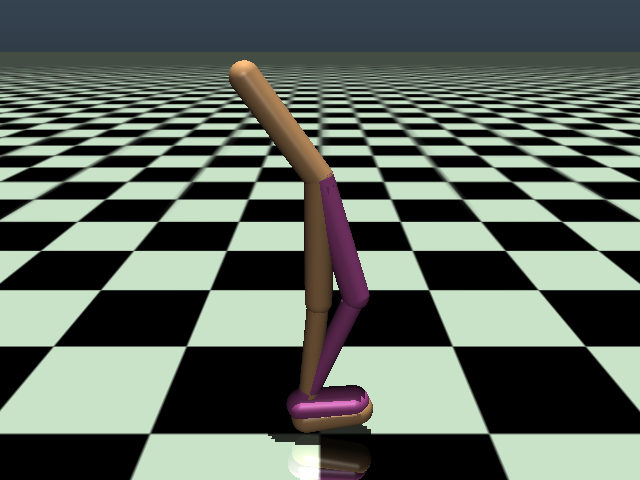}
    \includegraphics[height=1.8cm]{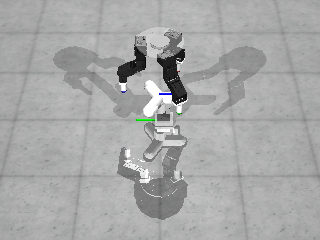}
    \end{center}
    \caption{{\footnotesize\textbf{Environments}: Our experiments included three tasks from OpenAI Gym and four tasks from ROBEL.}}
    \label{fig:envs-subset}
    \vspace{-1em}
\end{figure}

\paragraph{Gridworld for Fig.~\ref{fig:stochastic-gridworld}.}
This task is a $10 \times 10$ gridworld with obstacles shown in Fig.~\ref{fig:stochastic-gridworld}. There are four discrete actions, corresponding to moving to the four adjacent cells. With probability 50\%, the agent's action is ignored and a random action is taken instead. The agent starts in the top-left cell. The agent receives a reward of +0.001 at each time step and a reward of +10 when at the goal state. Episodes have 200 steps and we use a discount $\gamma = 0.9$.

\paragraph{Gridworld for Fig.~\ref{fig:aliasing}.}
This task is a $15 \times 15$ gridworld with obstacles shown in Fig.~\ref{fig:aliasing}. There are four discrete actions, corresponding to moving to the four adjacent cells. The dynamics are deterministic. The agent starts in the top-left cell. The reward is +1.0 at every state except the goal state, where the agent receives a reward of +100.0. We use $\gamma=0.9$ and compute optimal policies analytically using value iteration.
We implement the aliasing by averaging together the dynamics for each block of $3 \times 3$ states. Importantly, the averaging was done to the \emph{relative} dynamics (e.g., action 1 corresponds to move right) not the \emph{absolute} dynamics (e.g., action 1 corresponds to moving to state (3, 4)). To handle edge effects, we modified the averaged dynamics so that the agent could not exit the gridworld.
We computed the classifier analytically using the true and learned dynamics models. However, the augmented rewards become infinite because the learned model assigns non-zero probability to transitions that cannot occur under the true dynamics. This is not a failure of our theory, as the optimal policy would choose to never visit these states, but it presents a challenge for optimization. We therefore added label smoothing with parameter 0.7 to the classifier.

\paragraph{Gridworld for Fig.~\ref{fig:stochastic-gridworld}}
This task is a $10 \times 10$ gridworld with obstacles shown in Fig.~\ref{fig:stochastic-gridworld} \figright. There are four discrete actions, corresponding to moving to the four adjacent cells. With probability 90\%, the agent's action is ignored and a random action is taken instead. The agent starts in the top-left cell. The rewards depend on the Manhattan distance to the goal: transitions that lead away from the goal have a reward of +0.001, transitions that do not change the distance to the goal have a reward of +1.001, and transitions that decrease the distance to the goal have a reward of +2.001. We use $\gamma=0.5$ and compute values and returns analytically using value iteration.

\paragraph{Gridworld for Fig.~\ref{fig:lower-bounds}.}
This task used the same dynamics as Fig.~\ref{fig:stochastic-gridworld}. The one change is that the agent receives a reward of +1.0 at each time step and a reward of +10.0 at the goal state. We estimate the more complex lower bound (Eq.~\ref{eq:lb2}) by also learning the discount factor. However, since we currently do not have a method for learning non-Markovian dynamics to fully optimize this lower bound, we do not expect the lower bound to become tight.

\paragraph{\texttt{HalfCheetah-v2, Hopper-v2, Walker2d-v2}.} These tasks are taken directly from the OpenAI benchmark~\citep{brockman2016openai} without modification.

\paragraph{\texttt{metaworld-drawer-open-v2}.}
This task is based on the \texttt{drawer-open-v2} task from the Metaworld benchmark~\citep{yu2020meta}. To increase the difficulty of this task, we remove the reward shaping term (\texttt{reward\_for\_caging}) and just optimize the reward for opening the drawer (\texttt{reward\_for\_opening}).

\paragraph{\texttt{DClawPoseRandom-v0}, \texttt{DClawTurnRandom-v0}, \texttt{DClawScrewFixed-v0}, \texttt{DClawScrewRandom-v0}.}
These tasks are taken directly from the ROBEL benchmark~\citep{ahn2020robel} without modification.

\subsection{Differences between MnM and MnM-Approx}
Both MnM and MnM-Approx learn the same dynamics classifier, and both use the same GAN-like objective (Eq.~\ref{eq:model-objective}) to update the model. Both perform the same policy updates. The difference is the reward function used to update the Q-function:
\begin{align}
    \tilde{r}(s_t, a_t, s_{t+1}) &= (1 - \gamma) \log r(s_t, a_t) + \log \left( \frac{p(s_{t+1} \mid s_t, a_t)}{q(s_{t+1} \mid s_t, a_t)} \right) \tag{MnM} \\
    \tilde{r}(s_t, a_t, s_{t+1}) &= r(s_t, a_t). \tag{MnM-Approx}
\end{align}

\end{document}